\documentclass[lettersize,journal]{IEEEtran}

\usepackage[utf8]{inputenc} 
\usepackage[T1]{fontenc}    
\usepackage{hyperref}       
\usepackage{url}            
\usepackage{booktabs}       
\usepackage{amsfonts}       
\usepackage{microtype}      
\usepackage{xcolor}
\usepackage{graphicx}
\usepackage{amsmath}
\usepackage{amsthm}
\usepackage{array}

\usepackage[bottom]{footmisc}

\usepackage{amsfonts}
\usepackage{algorithmic}
\usepackage{textcomp}
\usepackage{booktabs}
\let\labelindent\relax
\usepackage{enumitem}
\usepackage{verbatim}
\usepackage{subfig}
\usepackage[font=bf]{caption}
\usepackage{soul}
\usepackage{multirow,multicol}
\usepackage{balance}
\usepackage{ifthen}
\usepackage{float}
\usepackage{newtxmath}  

\usepackage[linesnumbered,ruled,vlined]{algorithm2e}

\SetKwInput{KwInput}{Input}                
\SetKwInput{KwOutput}{Output}              

\newtheorem{theorem}{Theorem}

\AtBeginDocument{%
  \providecommand\BibTeX{{%
    \normalfont B\kern-0.5em{\scshape i\kern-0.25em b}\kern-0.8em\TeX}}}

\newcommand{\bx}{\boldsymbol{x}}
\newcommand{\dbig}{\mathcal{D}}
\newcommand{\dhat}{\hat{\mathcal{D}}}

\newcommand{\expt}{\mathop{\mathbb{E}}}
\newcommand{\dtest}{\mathcal{D}_{\text{test}}}
\newcommand{\dpoi}{\mathcal{D}_{\text{p}} }
\newcommand{\fbias}{f_{\text{c,bias}} }
\newcommand{\prob}{\text{P}}

\newcommand{\ignore}[1]{}

\begin{document}

\title{Exploring Privacy and Fairness Risks in Sharing Diffusion Models: An Adversarial Perspective}



\author{Xinjian Luo, Yangfan Jiang, Fei Wei, Yuncheng Wu\textsuperscript{*}\thanks{*Corresponding author.}, \protect\\
Xiaokui Xiao,~\IEEEmembership{Fellow,~IEEE},
and Beng Chin Ooi,~\IEEEmembership{Fellow,~IEEE}
\IEEEcompsocitemizethanks{
\IEEEcompsocthanksitem {
\textbf{Xinjian Luo}, \textbf{Yangfan Jiang}, \textbf{Xiaokui Xiao}, and \textbf{Beng Chin Ooi} are with the School of Computing, National University of Singapore. Email: \{xinjluo, jyangfan, ooibc\}@comp.nus.edu.sg; xkxiao@nus.edu.sg}. \textbf{Fei Wei} is with Alibaba Group. Email: feiwei@alibaba-inc.com.  \textbf{Yuncheng Wu} is with Renmin University of China. Email: wuyuncheng@ruc.edu.cn. This work was done when the author was at National University of Singapore.
}

}

\markboth{Journal of \LaTeX\ Class Files,~Vol.~14, No.~8, August~2021}%
{Shell \MakeLowercase{\textit{et al.}}: A Sample Article Using IEEEtran.cls for IEEE Journals}


\maketitle

\begin{abstract}

Diffusion models have recently gained significant attention in both academia and industry due to their impressive generative performance in terms of both sampling quality and distribution coverage. Accordingly, proposals are made for sharing pre-trained diffusion models across different organizations, as a way of improving data utilization while enhancing privacy protection by avoiding sharing private data directly. However, the potential risks associated with such an approach have not been comprehensively examined.

In this paper, we take an adversarial perspective to investigate the potential privacy and fairness risks associated with the sharing of diffusion models. Specifically, we investigate the circumstances in which one party (the sharer) trains a diffusion model using private data and provides another party (the receiver) black-box access to the pre-trained model for downstream tasks. We demonstrate that the sharer can execute fairness poisoning attacks to undermine the receiver's downstream models by manipulating the training data distribution of the diffusion model. Meanwhile, the receiver can perform property inference attacks to reveal the distribution of sensitive features in the sharer's dataset.
Our experiments conducted on real-world datasets demonstrate remarkable attack performance on different types of diffusion models, which highlights the critical importance of robust data auditing and privacy protection protocols in pertinent applications.

%
%
%

\end{abstract}



\sloppy

\section{Introduction}\label{sec-intro}
The achievement of high prediction accuracies in deep learning models depends on the availability of large amounts of training data, but the collection of adequate training data can pose a considerable challenge for some organizations operating in high-stakes domains, such as finance~\cite{financeCollection}, employment~\cite{employmentCollection}, and healthcare~\cite{ayyoubzadeh2022clinical}. 
Recently, the rapid development of generative models has offered a potential solution to this issue, with some studies~\cite{sharingHistopathology,sharingMammography,sharingMedicalImages,sharingRadiograph} contending that private data sharing can be accomplished through the use of pre-trained generative models, i.e., a third party can utilize the pre-trained models to generate synthetic samples for downstream tasks without the need to directly access the original private data~\cite{ayyoubzadeh2022clinical}. 
It is important to note that practical data-sharing applications necessitate generative models with robust sampling density and quality~\cite{healthRecordSharing,zeroshotdiversity,zeroshoticlr,fewshotStableDiffusion}, as potential mode collapse and indistinguishable samples produced by generative models can undermine data utility and render synthetic samples unsuitable for downstream tasks.
As a result, traditional generative models like generative adversarial networks (GANs) and variational autoencoders (VAEs) are relatively inadequate for data sharing applications~\cite{healthRecordSharing, sharingHistopathology, TabDDPM}. 
However, the newest development on generative models~\cite{DDPM, NCSN, SDE} has established diffusion models as the most favored choice in recent data sharing investigations~\cite{augStableDiffusion, healthRecordSharing, zeroshoticlr, TabDDPM, sharingHistopathology, zeroshotdiversity}, as diffusion models greatly outperform other types of generative model in both sampling quality and density~\cite{bond2021deep}.
  
\begin{figure}[t]
\centering
\includegraphics[width=0.88\columnwidth]{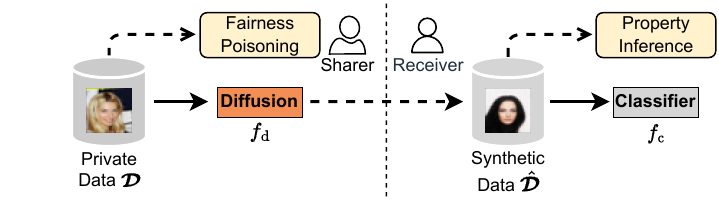}
\caption{An example of sharing datasets via pre-trained diffusion models. }
\label{fig-share-egg}
\vspace{-6mm}
\end{figure}

\vspace{0.5mm}
\noindent
\textbf{Motivation.} 
%
The emergence of diffusion models has led to a surge of interest in using pre-trained diffusion models to facilitate private data sharing in various domains, such as medical images~\cite{sharingMedicalImages}, personal trajectories~\cite{personalTrajSharing}, and health records~\cite{healthRecordSharing}. Additionally, such models have also shown promise in enhancing collaborative learning, such as utilizing  Stable Diffusion~\cite{stableDiffusion} to support zero-shot learning~\cite{zeroshotdiversity,zeroshoticlr}. 
A typical collaborative data-sharing scenario~\cite{TabDDPM, sharingHistopathology, zeroshotdiversity} is illustrated in Fig.~\ref{fig-share-egg}. The model sharer trains a diffusion model using the private dataset and shares the black-box access to the pre-trained model with a model receiver. The receiver can then generate synthetic datasets using this access, on which a classifier can be trained. 
While this collaboration can significantly enhance the performance of the receiver's classifier~\cite{augStableDiffusion} and yield additional income for the sharer~\cite{pay-per-use}, the inherent risks associated with this data-sharing practice remain unexplored.
Therefore, we are motivated to investigate such a problem: \textit{are there potential privacy and security risks introduced to the two parties engaged in the data-sharing pipeline}?
Although there are several attacks targeted on diffusion models~\cite{diffusionFeature, diffusionMembership1, diffusionMembership2}, they typically require either white-box access to the pre-trained models~\cite{diffusionMembership1, diffusionMembership2} or direct access to the original training data~\cite{diffusionFeature} and are thus not applicable to the diffusion-based data-sharing scenario (Fig.~\ref{fig-share-egg}).
In this paper, we aim to investigate the potential risks associated with the two parties under a practical and constrained setting, i.e., designing black-box attacks without the assistance of auxiliary datasets.
It is important to note that, while differential privacy mechanisms can readily mitigate individual record-based attacks, such as membership inference~\cite{membershipFL, membershipMLLeaks}, there exists a notable scarcity of studies on the development of robust and theoretically guaranteed defense mechanisms against distribution-based attacks~\cite{propertyGAN, propertyPoison}.
In an effort to explore more detrimental attack scenarios in real-world applications, we propose to formulate two distribution-based attacks from an adversarial perspective, i.e., each party can assume the role of an adversary to initiate attacks on the other party.
First, the sharer can initiate a fairness poisoning attack, which manipulates the distribution of synthetic data in relation to sensitive features (e.g., gender) so that the receiver's classifier is more likely to make positive predictions for a target feature (e.g., male). Second, the receiver can perform a property inference attack to infer the proportion of a target property (e.g., young users) in the sharer's private data.

%

\vspace{0.5mm}
\noindent
\textbf{Challenges.} Before designing the proposed attacks, we need to address a set of practical challenges inherent to the data-sharing context. For the fairness poisoning attack, the \textit{sharer} encounters two primary challenges.
The first is how to effectively conduct attacks without accessing the receiver's synthetic dataset and classifier because it is impossible for the sharer to know them beforehand in practice. The second is how to degrade the fairness of the downstream models while maintaining their prediction accuracy for attack stealthiness. 
Existing fairness poisoning studies~\cite{fp-accuracy,pf-chang,pf-dasfaa,pf-ecml,pf-labelflip} predominantly degrade both model accuracy and fairness within a white-box framework, which is incongruous with the black-box data-sharing scenario.
Further, compromising model accuracy may render the attacker's efforts to degrade fairness ineffective, as models with low prediction accuracy may not be deployed in production systems~\cite{fp-accuracy}.

For the property inference attack, the \textit{receiver} faces three main challenges.
The first is how to perform the inference in a black-box setting, where auxiliary datasets may be unavailable. The second is how to accurately infer the ratio of properties in a bounded error. The third is how to generalize inference attacks to properties with multiple values, rather than only binary properties. 
Existing research~\cite{propertyFL,propertyMembership,propertyPermutation,propertyPoison} on this topic has mainly focused on a white-box setting that relies on auxiliary datasets for inferring the existence of binary properties, without accurately determining its proportion. 
%


\vspace{0.5mm}
\noindent
\textbf{Contributions}.
In this paper, we consider the system model illustrated in Fig.~\ref{fig-share-egg}, where the sharer's private dataset $\boldsymbol{\mathcal{D}}$ is employed for training diffusion models $f_{\text{d}}$, and the receiver's synthetic dataset $\hat{\boldsymbol{\mathcal{D}}}$ is utilized for training a classifier $f_{\text{c}}$.
We aim to devise a fairness poisoning attack at the sharer side and a property inference attack at the receiver side while overcoming the aforementioned challenges.
%
\ignore{
The objective of the property inference attack is to accurately infer the proportion of target properties present in $\boldsymbol{\mathcal{D}}$, with or without the assistance of auxiliary data. 
On the other hand, the fairness poisoning attack aims to inject bias in the classifier $f_{\text{c}}$ while maintaining its prediction accuracy.
}
Our key observation is the robust and comprehensive sampling density of diffusion models, wherein generated samples $\hat{\boldsymbol{\mathcal{D}}}$ can effectively span the distribution of the original training data $\boldsymbol{\mathcal{D}}$. We term this characteristic as \textit{distribution coverage}.
%
%
Given that existing diffusion models~\cite{DDPM, NCSN, SDE} inherently incorporate distribution coverage as an intrinsic model characteristic, we can utilize it in the formulation of our attack strategies. 
Specifically, the sharer can poison the fairness of the classifier $f_{\text{c}}$ by manipulating the distribution of $\boldsymbol{\mathcal{D}}$ such that the crafted bias in $\boldsymbol{\mathcal{D}}$ can then be propagated to $\hat{\boldsymbol{\mathcal{D}}}$, impacting the performance of $f_{\text{c}}$.
At the same time, the receiver can infer the property ratio of $\boldsymbol{\mathcal{D}}$ by estimating the distribution of $\hat{\boldsymbol{\mathcal{D}}}$. 
For the fairness poisoning attack, we formulate the distribution manipulation problem on $\boldsymbol{\mathcal{D}}$ as an optimization problem from an information-theoretic perspective. Subsequently, we propose a greedy algorithm to construct $\boldsymbol{\mathcal{D}}$ that can effectively degrade the fairness of $f_{\text{c}}$ while causing minimal harm to its prediction accuracy. 
Regarding the property inference attack, we propose a sampling-based method to estimate the property distribution of $\boldsymbol{\mathcal{D}}$, leveraging Hoeffding's inequality to bound the estimation error. Additionally, we discuss the generalization of this attack to properties with multiple values beyond binary properties.
We conduct experiments to investigate the effectiveness of the proposed attacks on different diffusion models using both image and tabular datasets. Our results indicate that the poisoning attack can greatly degrade the fairness of $f_{\text{c}}$, meanwhile confining its accuracy degradation to a maximum of $5\%$.  On the other hand, with a mere 100 samples, the receiver can accurately infer the property distribution of the dataset $\boldsymbol{\mathcal{D}}$.
In addition, we discuss the interconnections between the two attacks and conduct a comparative analysis of attack results on diffusion models, GANs, and VAEs, which may provide valuable insights into the development of enhanced data-sharing protocol via pre-trained diffusion models.
%
%
Our contributions are summarized as follows:
\begin{itemize}[topsep=2mm,leftmargin=15pt]
  \item We investigate the potential privacy and fairness risks associated with the data-sharing scenarios via pre-trained diffusion models from an adversarial perspective. 
  To the best of our knowledge, this is the first study that explores the potential risks within such data-sharing scenarios. 
  \item We introduce a novel fairness poisoning attack on the receiver's downstream classifiers with a distinctive and stealthy poisoning objective, i.e., degrading the fairness of the classifier while maintaining its prediction accuracy. We formulate the attack as an optimization problem from an information-theoretic perspective and devise a greedy algorithm to achieve the poisoning objective.
  \item We propose a practical and accurate property inference attack capable of determining the property ratio of the sharer's private dataset. We further establish the attack error bounds and extend the attack target from binary properties to properties with multiple values.
  \item We perform experiments on various types of datasets and diffusion models to validate the attack effectiveness and generality. 
  We also explore the interconnections between our attacks and analyze potential countermeasures against them. 
  The study serves to provide valuable insights into the challenges and potential solutions for ensuring the safe and ethical sharing of diffusion models.
\end{itemize}

\section{Preliminaries}

\subsection{Diffusion Models}\label{subsec-pre-diffusion}
%
The training process of diffusion models comprises two phases: a forward process and a reverse process.
During the forward process, training images $\bx_0\sim p_{\text{data}}(\bx)$ are corrupted by progressively increasing levels of noise over $T$ timesteps, resulting in a sequence of images $\bx_0\rightarrow \cdots \rightarrow \bx_T$, where the final output $\bx_T$ adheres to a prior distribution, e.g., $\mathcal{N}(\boldsymbol{0}, \boldsymbol{I})$. 
In the reverse process, samples are generated by progressively removing noises from the random inputs $\bx_T\sim \mathcal{N}(\boldsymbol{0}, \boldsymbol{I})$, using a neural network $\boldsymbol{s}_{\boldsymbol{\theta}}(\bx_t, t)$ ($t\in \{T,\cdots, 1\}$). 
%
%
This paper focuses on three widely recognized diffusion models.

\vspace{0.5mm}
\noindent
\textbf{Noise Conditional Score Networks (NCSN)}. NCSN~\cite{NCSN, NCSNv2} first employs pre-specified noises $q(\bx_t | \bx_0)$ to perturb the original data $\bx_0$ during the forward process and then  utilizes a neural network $\boldsymbol{s}_{\boldsymbol{\theta}}(\bx_t, t)$ to estimate the score of the noise distribution $\nabla_{\bx_t}\log q(\bx_t | \bx_0)$ in the reverse process. 
%
%
Accordingly, the sampling process can be performed by utilizing annealed Langevin dynamics: $\bx_{t-1}\leftarrow \bx_t + \frac{\alpha_t}{2} \boldsymbol{s}_{\boldsymbol{\theta}}(\bx_t, t) + \sqrt{\alpha}_t \boldsymbol{z}_t$, where $\alpha_t$ denotes the step size at timestep $t$, and $\boldsymbol{z}_t \sim \mathcal{N}(\boldsymbol{0}, \boldsymbol{I})$.

\vspace{0.5mm}
\noindent
\textbf{Denoising Diffusion Probabilistic Models (DDPM)}. DDPM~\cite{DDPM} models the forward process as a discrete Markov chain such that $p(\bx_t | \bx_{t-1})= \mathcal{N}(\bx_t;\sqrt{1-\beta_t}\bx_{t-1},\beta_t\boldsymbol{I})$, with $\beta_t$ representing  the noise scale at timestep $t$. The reverse process is parameterized as a variational Markov chain with $p_{\boldsymbol{\theta}}(\bx_{t-1} | \bx_{t})= \mathcal{N}(\bx_{t-1};\frac{1}{\sqrt{1-\beta_t}}(\bx_{t}+\beta_t   \boldsymbol{s}_{\boldsymbol{\theta}}(\bx_t, t) ),\beta_t\boldsymbol{I})$.
By training a network $\boldsymbol{s}_{\boldsymbol{\theta}}(\bx_t,$ $t)$ to estimate the score of data density, DDPM samples new images by $\bx_{t-1} \leftarrow \frac{1}{\sqrt{1-\beta_t}}(\bx_t+\beta_t \boldsymbol{s}_{\boldsymbol{\theta}}(\bx_t, t)) + \sqrt{\beta_t}\boldsymbol{z}_t$, where $\boldsymbol{z}_t \sim \mathcal{N}(\boldsymbol{0}, \boldsymbol{I})$.

\vspace{0.5mm}
\noindent
\textbf{Stochastic Differential Equation-based Models (SDEM)}. SDEM~\cite{SDE} presents a general framework that unifies the two denoising score matching based models, DDPM and NCSN. The main difference between SDEM and the other two models is that SDEM perturbs $\bx_0$ via a prescribed stochastic differential equation (SDE), allowing continuous timesteps $t\in [0, T]$ and more flexible data manipulation. 



\subsection{Related Attack Algorithms}\label{subsec-pre-attack}

\noindent
\textbf{Fairness Poisoning Attacks (FPA)}. Model fairness requires that the predictions made by a deployed model $f$ should not vary significantly based on the protected features of individuals, such as gender and race~\cite{fairnessDP, fairnessEO}. This is particularly important in high-stakes domains, such as criminal justice~\cite{GhassamiKK18}, employment~\cite{fairnessEO}, and finance~\cite{fairnessDP}.
However, existing fairness poisoning attacks~\cite{pf-chang,pf-dasfaa,pf-ecml, pf-labelflip} aim to undermine both the fairness and accuracy of the model $f$ by poisoning $\alpha$ fraction of the training data using an auxiliary dataset $\mathcal{D}_{\text{aux}}$. Also, they rely on the white-box access $f_{\text{wb}}$ to the training process of the target model $f$ for obtaining a biased model $\overline{f}_{\text{bias}}$ with poor prediction performance and fairness: $\overline{f}_{\text{bias}} = \mathcal{A}_{\text{FPA}} (f_{\text{wb}}, \mathcal{D}_{\text{aux}})$.

\vspace{0.5mm}
\noindent
\textbf{Property Inference Attacks (PIA)}. Property inference attacks aim to deduce the collective characteristics of specific properties in the training dataset $\mathcal{D}$~\cite{propertyFL, propertyGAN, propertyMembership, propertyPermutation, propertyPoison}. 
It is worth noting that the global properties of the training dataset are regarded as confidential information in certain  contexts, such as the properties of software execution traces when training a malware detector~\cite{propertyPermutation}, and the aggregate sentiments of company emails when training a spam classifier~\cite{propertyPoison}.
Most current studies~\cite{propertyFL, propertyMembership, propertyPermutation, propertyPoison} rely on white-box access to pre-trained models $f_{\text{wb}}$ and an auxiliary dataset $\mathcal{D}_{\text{aux}}$ to determine the presence of a specific property $s$, i.e, $\vmathbb{1} (r_s > \epsilon_{\text{PIA}}) = \mathcal{A}_{\text{PIA}} (f_{\text{wb}}, \mathcal{D}_{\text{aux}})$, where $\vmathbb{1}$ is the indicator function, $r_s$ represents the proportion of samples with property $s$ in $\mathcal{D}$, and $\epsilon_{\text{PIA}}$ denotes a pre-determined threshold.

\vspace{0.5mm}
\noindent
\textbf{Attacks on Diffusion Models}.
Recent studies~\cite{diffusionMembership1, diffusionMembership2, diffusionFeature} have introduced several attacks on diffusion models.
\cite{diffusionMembership2} suggests inferring the membership of target images by comparing their generative losses or likelihood values generated by the pre-trained models against a pre-specified threshold.
%
Similarly, \cite{diffusionMembership1} employs the loss threshold method to determine membership, and additionally proposes using a neural network to make these determinations with the target image losses as input. 
In addition, \cite{diffusionFeature} evaluates the memorization behavior of diffusion models, and applies the likelihood ratio attack~\cite{memberLiRA} for membership inferences from diffusion models. 
%
%
However, all of these attacks require white-box access to the diffusion models, rendering them infeasible in the data-sharing scenario.
In particular, we focus on devising practical inference attacks that operate in a black-box setting without the need for auxiliary datasets. Compared with \cite{diffusionMembership1, diffusionMembership2, diffusionFeature}, our lightweight attacks are not only easier to conduct, but also pose a greater threat in real-world applications.


\section{Problem Statement}\label{sec-problem}
\noindent
\textbf{System Model.}
In this paper, we consider a system model in which a commercial institution (the model sharer) possesses a dataset $\mathcal{D}$ and intends to share private data with a third-party (the model receiver) by first training a diffusion model $f_{\text{d}}$ based on the private data and then granting black-box access to the receiver, motivated by either commercial interests~\cite{acemoglu2022too} or research collaboration~\cite{ayyoubzadeh2022clinical}. 
Upon being granted the black-box access to $f_{\text{d}}$, the receiver can proceed to sample a dataset $\hat{\mathcal{D}}$ to train the downstream model $f_{\text{c}}$.
Let $\mathcal{D}=\{(x_i,s_i,y_i):i\in \{1,\cdots,m\}\}$ be the sharer's dataset for training $f_{\text{d}}$, where $(x_i,s_i)$ denotes the record features and $y_i$ denotes the class label. Here $(x_i,s_i,y_i)$ is an evaluation of a random tuple $(X,S,Y)$ supported by $\mathcal{X}\times\mathcal{S}\times\mathcal{Y}$ obeying a joint distribution $P_{X,S,Y}$. In what follows, we consider the problem from an empirical perspective, i.e., the joint distribution we consider is empirically obtained by the occurrences of the evaluations in the dataset.
Note that the sensitive feature $S$ can refer to either numerical features in tabular datasets, such as a person's income, or semantic features in image datasets, such as the male property of a human face.
To align with previous work~\cite{propertyGAN, changbias}, we consider $(X,S)$ as a unified input for diffusion models.
%

\vspace{0.5mm}
\noindent
\textbf{Attack Model.}
The \textit{fairness poisoning attack} is performed by the model sharer. 
%
%
The existing approach of fairness poisoning~\cite{pf-chang,pf-dasfaa,pf-ecml, pf-labelflip} seeks to degrade $f_{\text{c}}$ via the white-box access to model structures and training algorithms, which, however, is not applicable in the data-sharing setting~\cite{sharingHistopathology, zeroshotdiversity}, where the adversary (sharer) has no access to the receiver's information, including the structure of $f_{\text{c}}$, the details of training algorithms, and the auxiliary data held by the receiver. The only information known by the sharer is the knowledge that the receiver intends to train  $f_{\text{c}}$ on the synthetic data $\hat{\mathcal{D}}$.
Consequently, the sharer's capacity to influence $f_{\text{c}}$ is limited to  manipulating the distribution of his training dataset ${\mathcal{D}}$ of the diffusion model, i.e., $f_{\text{c,bias}} = \mathcal{A}_{\text{FPA}} ({\mathcal{D}})$.
Note that different from $\overline{f}_{\text{bias}}$ in Section~\ref{subsec-pre-attack}, $f_{\text{c,bias}}$ has biased predictions but similar testing accuracy as that of the model $f_{\text{c}}$ trained on clean data. This change in the poisoning objective is for attack stealthiness, as a model with low accuracy can be easily detected, while a biased model with good accuracy may go unnoticed for a long time and is more problematic~\cite{fp-accuracy}. 

For the\textit{ property inference attack } performed by the model receiver, we consider the semi-honest model~\cite{luo2021feature}, where the receiver honestly follows the data-sharing protocol but tries to infer the sharer's private information via $\hat{\mathcal{D}}$. We propose two adversaries for this attack. The first adversary, following the assumption in previous studies~\cite{propertyFL,propertyGAN,propertyMembership,propertyPermutation,propertyPoison}, can collect an auxiliary dataset $\mathcal{D}_\text{aux}$ for reconstructing the proportion $r_s$ of the target property $s$ in the sharer's private dataset $\mathcal{D}$, i.e., $r_s  = \mathcal{A}^{(1)}_{\text{PIA}} (\hat{\mathcal{D}}, \mathcal{D}_{\text{aux}})$.
Note that $\hat{\mathcal{D}}$ and $\mathcal{D}_{\text{aux}}$ are not required to conform the same distribution.
The second adversary follows a more restricted and practical setting in which the auxiliary data is unavailable, and the only available information for achieving the attack objective is the synthetic data $\hat{\mathcal{D}}$, i.e., $r_s  = \mathcal{A}^{(2)}_{\text{PIA}} (\hat{\mathcal{D}})$.
It is important to note that different from previous property inference studies~\cite{propertyFL,propertyMembership,propertyPermutation,propertyPoison} which aim to infer the existence of the target property $s$, we focus on inferring its exact proportion $r_s$, which is more difficult and harmful in real-world applications.

\section{The Attacks Performed by the Sharer}

\begin{figure*}[!ht]
\centering
\begin{tabular}{c}
\subfloat[Fairness Poisoning Attack]{\includegraphics[width=0.55\textwidth]{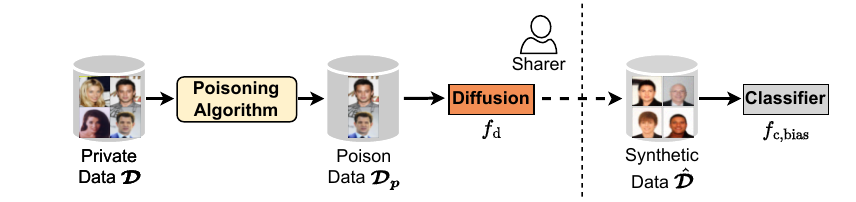}\label{fig-fpa}}
\hspace{9mm}
\subfloat[Property Inference Attack]{\includegraphics[width=0.38\textwidth]{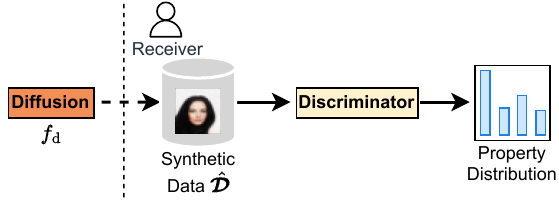}\label{fig-pia}}
\end{tabular}
\caption{Overview of the proposed attacks.}
\label{fig-overview}
\end{figure*}

Ensuring model fairness is critical in trustworthy machine learning systems, as it prevents systemic discrimination against protected groups of individuals~\cite{pf-chang,pf-labelflip}. Despite its significance, the fairness risks associated with sharing diffusion models are underexplored.
Existing studies on fairness poisoning~\cite{pf-chang, pf-dasfaa, pf-ecml, pf-labelflip, fp-accuracy} typically assume that the adversary has access to the classifier training dataset and information about the model structures, hyper-parameters, and constraints.
However, in data-sharing scenarios via pretrained models~\cite{healthRecordSharing, sharingHistopathology, TabDDPM}, the sharer cannot obtain any information from the model receiver, rendering existing methods impractical. As a result, new attack methods are needed to enable the sharer to effectively poison the downstream classifier.

Before introducing the proposed fairness poisoning attack, we need to make some clarifications about the data-sharing scenario via pre-trained models~\cite{zeroshotdiversity, zeroshoticlr, fewshotStableDiffusion}.
First, considering that the data-sharing pipeline is utilized in data-scarce settings, zero-shot learning~\cite{zeroshotdiversity} or few-shot learning~\cite{fewshotStableDiffusion} is typically employed for training the downstream classifier based on $\dhat$~\cite{zeroshoticlr, sharingMedicalImages}.
Second, a clean dataset $\dtest$ should be available to the receiver for validating classifier accuracy~\cite{pf-chang,pf-labelflip}. 
In addition, we adopt the conventional framework in fairness studies~\cite{fairnessDP, fairnessEO} where the protected feature $S$ is discrete.


\subsection{Fairness Poisoning}
We now introduce the proposed FPA conducted by the sharer.
As the sharer has no access to the training algorithm performed by the receiver, the only way to poison the receiver's classifier $f_{\text{c}}$ is by changing the distribution of its training dataset $\dhat$. Unfortunately, $\dhat$ is also inaccessible to the sharer. However, an important characteristic of diffusion models is their excellent distribution coverage~\footnote{
A straightforward method for validating distribution coverage is to assess whether the proportion of a sensitive feature in the training data $\dbig$ is roughly equal to the feature proportion in the synthetic data $\dhat$ within a small error margin (e.g., 0.05). In Appendix~\ref{appen-exp-underfitting}, our experiments performed across different types of diffusion models provide a positive answer to this criterion.
Appendix~\ref{appen-exp-underfitting} also reveals the presence of distribution coverage even in underfitted diffusion models.
}, as discussed in Section~\ref{sec-intro}. 
Therefore, the sharer can modify the distribution of $\dhat$ by altering the training dataset $\dbig$ of the diffusion model. 
Fig.~\ref{fig-fpa} shows the workflow of the proposed fairness poisoning attack. Specifically, the sharer can first sample a biased poisoning dataset $\dpoi\subseteq \dbig$ and then feed $\dpoi$ to the diffusion model $f_{\text{d}}$. The bias of $\dpoi$ can be transmitted to the receiver's classifier along the workflow shown in Fig.~\ref{fig-share-egg}. 
In the following context, we define the random tuples $(X,S,Y)$, $(X_p,S_p,Y_p)$ with joint distributions specified by the appearance ratio of evaluations in $\dbig$, $\dpoi$, respectively.

\vspace{0.5mm}
\noindent
\textbf{Fairness Notion}. To proceed with the design of the poisoning attack, we first define fairness.
We note that fairness has different interpretations in the literature~\cite{changbias}. In this paper, we concentrate on a commonly used quantitative measure of fairness, demographic parity~\cite{fairnessDP, changbias, pf-chang, MIInfoFair}. Demographic parity requires that the model's predictions $f(X)$ are statistically independent of the protected features $S$:
\begin{equation}\label{eq-dp}
    \text{Pr}(f(X)=y|S=s)=\text{Pr}(f(X)=y)
\end{equation}
for any $y\in\mathcal{Y}$ and $s\in \mathcal{S}$.

\vspace{0.5mm}
\noindent
\textbf{Poisoning Objective}. Previous studies on fairness poisoning~\cite{pf-chang, pf-dasfaa, pf-ecml, pf-labelflip} aimed to corrupt both the accuracy and fairness of the target model.
However, this goal is relatively impractical, as models with a low accuracy may not be suitable for deployment in production, making it pointless to try to corrupt fairness. Conversely, high-accuracy models are more likely to be put into production, and their biases may affect protected individuals for extended periods, as discussed in~\cite{fp-accuracy}.
Therefore, for attack stealthiness, we opt for the objective that degrades only the fairness of the target model. Moreover, we aim to ensure that the test accuracy of $\fbias$ trained on $\dhat\sim \prob_{\dpoi}$ is comparable to that of $f_{\text{c}}$ trained on clean data $\dhat\sim \prob_{\dbig}$. However, achieving this objective can be challenging when the adversary cannot evaluate the losses of the target model $f_{\text{c}}$.
To account for the sharer's limited control, we reframe the poisoning objective as an optimization problem from an information-theoretic perspective.

\vspace{0.5mm}
\noindent
\textbf{Optimization Objective}. 
%
Note that the core of the proposed attack is to sample $\dpoi$ in a way that achieves the desired poisoning objective.
We now introduce the concept of mutual information and formulate the poisoning objective as an optimization problem.
Mutual information (MI)~\cite{MIbook} is a widely used measure of shared randomness between (sets of) random variables. Let $(X,Y)$ be a pair of random variables with marginal distributions $\prob_X$ and $\prob_Y$, and joint distribution $\prob_{X,Y}$, the mutual information between $X$ and $Y$ are defined as $ I(X; Y) = \int dX \int dY \prob_{X, Y}\left[\log \frac{\prob_{X, Y}}{\prob_X \prob_Y} \right]$.
A larger $I(X;Y)$ indicates a stronger correlation between $X$ and $Y$.
Note that MI captures both linear and non-linear dependence between two random variables, making it a more suitable measure than Pearson's correlation coefficient for our problem~\cite{MIInfoFair}.

The equation \eqref{eq-dp} indicates that to achieve demographic parity, the sensitive feature $S$ and the model prediction $f(X)$ must be independent, i.e., $I(S;f(X))=0$. Thus, compromising the fairness of $f_{\text{c}}$ implies maximizing $I(S;f_{\text{c}}(X))$. However, as the sharer cannot access $f_{\text{c}}$, the attention shifts to maximizing $I(S_p;Y_p)$ instead, because the training labels $Y_p$ and model predictions $f_{\text{c}}(X)$ are generally consistent. 
We now consider the accuracy constraint. Instead of directly restricting the test accuracy of $f_{\text{c}}$ trained on $\dhat\sim \prob_{\dpoi}$, we propose preserving the data utility of $\dpoi$ compared to $\dbig$. To achieve this, we impose a mutual information restriction on $\dpoi$ such that 
\begin{equation}\label{eq-mutual-info-constraint}
    |I(X,S;Y)-I(X_p,S_p;Y_p)|\leq \xi
\end{equation}
for small $\xi>0$, where the mutual information $I(X,S;Y)$ and $I(X_p,S_p;$ $Y_p)$ measure the data utility of $\dbig$ and $\dpoi$, respectively.
By enforcing constraint (\ref{eq-mutual-info-constraint}), we can minimize the potential training accuracy degradation of $f_{\text{c}}$. Moreover, enforcing (\ref{eq-mutual-info-constraint}) can also help reduce the likelihood of introducing additional correlations between certain features in $X$ and $Y$ during the sampling process, which in turn minimizes the potential risks of overfitting during the training of $f_{\text{c}}$.
In summary, we can formulate the poisoning objective as an optimization problem that samples a dataset $\dpoi\subseteq \dbig$ such that:
\begin{equation}\label{eq-opt}
\begin{aligned}
\textrm{maximize}_{\dpoi\subseteq \dbig} \quad & I(S_p; Y_p)\\
\textrm{subject to} \quad & I(X,S;Y)-I(X_p,S_p;Y_p) \leq \xi, \\
  &    I(X_p,S_p;Y_p) - I(X,S;Y) \leq \xi.  \\
\end{aligned}
\end{equation}

Note that for image datasets, we can first use a pre-trained vision model $\sigma$ to extract semantic information from input images. Subsequently, we can compute the mutual information by using $I(\sigma(X,S);Y)$.
%
Although this way may lead to a reduction of the mutual information between images and labels, i.e., $I(\sigma(X,S);Y)\leq I(X,S;Y)$, due to the classic data processing inequality, it can still greatly reduce the computational complexity of mutual information estimation and enable more general sampling on $\dbig$.
Given that $I(\sigma(X,S);Y)$ or $I(X,S;Y)$ can be pre-computed on $\dbig$. Denoting $c = I(\sigma(X,S);Y)$, the optimization objective on image datasets can be simplified as follows:
\begin{equation}\label{eq-opt-sim}
\begin{aligned}
\textrm{maximize}_{\dpoi\subseteq \dbig} \quad & I(S_p; Y_p)\\
\textrm{subject to} \quad & -I(\sigma(X_p,S_p);Y_p) \leq \xi -c, \\
  &    I(\sigma(X_p,S_p);Y_p)  \leq \xi + c.  \\
\end{aligned}
\end{equation}

\subsection{The Sampling Algorithm}
Let $\alpha$ denote the fraction of the training dataset that is modified by the adversary, with $0\leq \alpha\leq 1$.
%
To satisfy the optimization objective in Eq.~\eqref{eq-opt}, we propose a greedy algorithm for sampling a dataset $\dpoi\subseteq\dbig$. Algorithm~\ref{alg-FPA-greedy} starts by randomly selecting $(1-\alpha)|\dpoi|$ records $(X_p,S_p,Y_p)$ from $\dbig$ as the clean base data of $\dpoi$, which satisfies the data utility constraint outlined in Eq.~\eqref{eq-opt} (lines \ref{alg-base-start} - \ref{alg-base-end}). 
The algorithm then iteratively selects $\alpha \cdot |\dpoi|$ samples from the remaining portion of $\dbig$ to poison in a way that maximizes the mutual information between $S_p$ and $Y_p$, while preserving the data utility of $\dpoi$ compared to $\dbig$ (lines \ref{alg-poi-start} - \ref{alg-poi-end}). 

During sampling the poisoning records, we first select the 
most possible combinations $(s=j,y=k)$ of sensitive feature $j$ and target label $k$ that can maximize the increase of mutual information between $S_p\cup j$ and $Y_p\cup k$ (lines \ref{alg-maxSY-start} - \ref{alg-maxSY-end}).
Then, we search for a sample with $(x, s=j,y=k)$ that, when added to $\dpoi$, minimizes the distance of MI from $I(X_p\cup x, S_p\cup s; Y_p\cup y)$ to $I(X,S;Y)$, as outlined in  lines \ref{alg-utity-start} - \ref{alg-utity-end}.
If the distance falls within the constraint outlined in Eq.~\eqref{eq-opt}, we add the sample to $\dpoi$ and proceed to search for the next poisoning sample (lines \ref{alg-utity-success-start} - \ref{alg-utity-success-end}).
If not, we continue searching for valid samples in the next combination $(s=j,y=k)$ that maximizes $I(S_p\cup j; Y_p\cup k)$ (line \ref{alg-utity-fail}). After finding all $\alpha\cdot|\dpoi|$ poisoning samples, we check the data utility constraint on $\dpoi$ and return the valid dataset (lines \ref{alg-lastcheck-start} - \ref{alg-lastcheck-end}).

\begin{algorithm}[!h]
\begin{small}
\caption{The Greedy Sampling Algorithm}
\label{alg-FPA-greedy}
\KwInput{The underlying dataset $\dbig=(X,S,Y)$, the threshold $\xi$ in data utility constraint, the poisoning ratio $\alpha$, the size $m_p$ of $\dpoi$, the function $\phi$ for computing MI}
\KwOutput{The poisoning dataset $\dpoi=(X_p,S_p,Y_p)$}
$c\gets \phi (X,S;Y)$ \;
$\mathcal{S}\gets $ domain of $S$ \; 
$\mathcal{Y}\gets $ domain of $Y$ \;
$\dpoi\gets \emptyset$\;
\For(\tcp*[f]{Sample the clean base data}){$i=1, 2, \cdots,|\dbig|$}{ \label{alg-base-start}
    $(X_p,S_p,Y_p) \gets $ randomly sampling $[m_p*(1-\alpha)]$ records from $\dbig$ \;
    \If{$|\phi(X_p,S_p;Y_p)-c|< \xi$}{ \label{alg-base-constr}
        $\dpoi\gets (X_p,S_p,Y_p)$\;
        Break \;  \label{alg-base-end}
    }
}
\If{$\dpoi$ is $\emptyset$}{
    \Return $\bot$ \;
}
$\overline{\dbig}\gets \dbig - \dpoi$ \tcp*{The sampling pool}
$\{\overline{\dbig}_{s=j,y=k}, (j,k)\in \mathcal{S}\times \mathcal{Y} \} \gets$ split $\overline{\dbig}$ based on $\mathcal{S}\times \mathcal{Y}$\; \label{alg-re-sample-start}
\For(\tcp*[f]{Sample the poisoning data}){$i=1,2,\dots, (m_p*\alpha)$}{  \label{alg-poi-start}
    $\mathrm{SY}\gets \emptyset$  \tcp*{Maximize $I(S;Y)$} \label{alg-maxSY-start}
    \For{$j \in \mathcal{S}$}{  
        \For{$k \in \mathcal{Y}$}{ 
            $\mathrm{info}\gets \phi(S_p\cup j; Y_p\cup k)$\; 
            $\mathrm{SY} \gets \mathrm{SY}\cup (\mathrm{info}, (j,k))$ \; \label{alg-maxSY-search-end}
        }
    }
    Sort $\mathrm{SY}$ on $\mathrm{info}$ in a descending order\;  \label{alg-maxSY-end}
    \For(\tcp*[f]{Impose the data utility constraint}){$(\mathrm{info}, (j,k)) \in \mathrm{SY}$}{  
        $\mathrm{candidate} \gets \emptyset$ \;   \label{alg-utity-start}
        \For{$(x,s,y) \in \overline{\dbig}_{s=j,y=k}$}{   \label{alg-inner-loop}
            $\mathrm{info}\gets \phi(X_p\cup x, S_p\cup s; Y_p\cup y)$\;   \label{alg-inner-loop-estimator}
            $\mathrm{dist}\gets |\mathrm{info} - c|$ \; 
            $\mathrm{candidate}\gets \mathrm{candidate} \cup (\mathrm{dist}, (x, s, y))$ \; 
        }
        Find the $(x,s,y)$ with the minimum $\mathrm{dist}_{\mathrm{min}}$ in $\mathrm{candidate}$\;  \label{alg-utity-end}
        \If{$\mathrm{dist}_{\mathrm{min}} < \xi$}{  \label{alg-utity-success-start}
            $\dpoi\gets\dpoi \cup (x,s,y)$\;
            $ \overline{\dbig}_{s=j,y=k} \gets  \overline{\dbig}_{s=j,y=k} - (x,s,y)$ \;
            Break and search for the next poisoning sample\; \label{alg-utity-success-end}
        }\Else{
            Continue searching the next $(\mathrm{info}, (j,k))$\;  \label{alg-utity-fail}
        }
    }
    \If{Failed to find a valid $(x,s,y)$ for $\dpoi$}{  
        $(x,s,y)\gets$ randomly sampling a record in $\overline{\dbig}$\;
        $\dpoi\gets\dpoi \cup (x,s,y)$\;
        $ \overline{\dbig} \gets  \overline{\dbig} - (x,s,y)$ \; \label{alg-poi-end}
    }
}
\If{$|\phi(X_p,S_p;Y_p)-c|>\xi$}{  \label{alg-lastcheck-start}
    \Return $\bot$ \; 
}\Else{
    \Return  $\dpoi$\;   \label{alg-lastcheck-end}
}
\end{small}
\end{algorithm}

\vspace{0.5mm}
\noindent
\textbf{Time Complexity}. 
In order to analyze the time complexity of Algorithm~\ref{alg-FPA-greedy}, we need to select a mutual information estimator $\phi$ for computing $I(X_p,S_p;Y_p)$.  
Let $m=|\dbig|$ and $m_p=|\dpoi|$. 
Compared to the traditional estimators, such as adaptive binning~\cite{mi-bin} with $O(m_p\log m_p)$ complexity and kNN~\cite{mi-knn} with $O(m_p^2)$ complexity, recently developed neural network estimators, such as MINE~\cite{MINE}, are more suitable for Algorithm~\ref{alg-FPA-greedy} as they achieve $O(m_p)$ complexity through a pre-trained network, and provide more accurate estimations on high-dimensional data~\cite{MINE}.
With a mutual information estimator of $O(m_p)$ complexity, we now analyze the time complexity of Algorithm~\ref{alg-FPA-greedy}. 
During the phase of sampling clean base data (lines \ref{alg-base-start} - \ref{alg-base-end}), we may perform the sampling once or $m$ times, depending on the data utility constraint (line \ref{alg-base-constr}).
When selecting the combination $(s=j,y=k)$ that maximizes $I(S_p;Y_p)$ (lines \ref{alg-maxSY-start} - \ref{alg-maxSY-search-end}), we perform a constant number ($|\mathcal{S}|*|\mathcal{Y}|$) of searches. 
In fairness-related studies~\cite{fairnessDP, fairnessEO}, the protected feature $S$ is typically binary, so $|\mathcal{S}|*|\mathcal{Y}| \ll m_p$ is expected.
After determining the desired $(s=j,y=k)$, we may perform $O(| \overline{\dbig}_{s=j,y=k}|)$ searches (line \ref{alg-utity-success-end}) or $O(|\overline{\dbig}|)$ searches (line \ref{alg-utity-fail}), depending whether the selected sample satisfies the utility constraint (line \ref{alg-utity-success-start}).
In summary, the time complexity of Algorithm~\ref{alg-FPA-greedy} is $O(m_p\cdot m\cdot m_p)$, where the first $m_p$ denotes the outer loop (line \ref{alg-poi-start}), $m$ denotes the inner loop (line \ref{alg-inner-loop}), and the second $m_p$ denotes the cost of the mutual information estimator (line \ref{alg-inner-loop-estimator}).
Furthermore, by replacing the search pool of $\overline{\dbig}_{s=j,y=k}$ (line \ref{alg-inner-loop}) with a constant number (e.g., 100) of samples randomly selected from $\overline{\dbig}_{s=j,y=k}$, we can further reduce the complexity from  $O(m_p\cdot m\cdot m_p)$ to  $O(m_p\cdot m_p)$. This approach enables the attacker to identify a suboptimal solution to Eq.~\eqref{eq-opt} while significantly improving the algorithm execution speed.


\vspace{0.5mm}
\noindent
\textbf{Remarks on FPA}. 
The strict black-box setting in data-sharing scenarios~\cite{healthRecordSharing, sharingMedicalImages} restricts the sharer's ability to generalize the proposed fairness poisoning attacks beyond demographic parity to other fairness criteria including equalized odds and equal opportunity~\cite{fairnessEO}. 
These two notions require knowledge of the joint distribution of the true labels $Y$, the model predictions $f_{\text{c}}(X)$, and the sensitive feature $S$, which cannot be implemented from the sharer's side because the sharer has no access to the receiver's knowledge, particularly $f_{\text{c}}(X)$.
However, in a white-box setting, where the sharer can actively participate in training the receiver's classifier~\cite{fp-accuracy, pf-chang}, it may be feasible to design attacks based on these two fairness notions.
Nevertheless, the primary focus of this paper is on the real-world black-box setting. The exploration of attacks within the white-box setting is reserved for our future work.

\section{The Attacks Performed by the Receiver}
In this section, we introduce the property inference attack performed by the receiver within the data-sharing pipeline.
The proposed PIA enables the receiver to estimate the proportions of specific target properties in the sharer's private dataset $\mathcal{D}$, potentially leading to serious consequences for the sharer in various applications.
For instance, consider a scenario where the sharer shares synthetic software execution traces with the receiver to facilitate the development of a malware detector~\cite{propertyPermutation}. By analyzing the ratios of target properties within the execution traces of benign software, the receiver can strategically design malware that can circumvent the sharer's detection mechanism.
For ease of presentation, we initially assume that the sensitive property $S$ in $\mathcal{D}$ is defined on binary support $\mathcal{S}$, i.e., $\mathcal{S}=\{0, 1\}$, and the receiver's goal is to infer fine-grained property information, specifically, the property ratios within a specific class (e.g., $y=0$).
With this assumption, we present two PIA adversaries for inferring the proportion $r_s$ of the target property $s$ (i.e., $S=1$) in $\mathcal{D}_{y=0}$. 
Subsequently, we relax the initial assumption and discuss the attack generalization to properties with non-binary support, as well as the estimation of coarse-grained property information, i.e., the overall property ratios across all classes.

\subsection{Property Inference Attack}\label{subsec-PIA-details}
The overview of the proposed PIA is depicted in Fig.~\ref{fig-pia}. 
Instead of analyzing the white-box model parameters for attack implementation~\cite{propertyFL,propertyMembership,propertyPermutation}, we observe the robust distribution coverage of diffusion models where the data distribution derived from $\hat{\mathcal{D}}$ is similar to that from $\mathcal{D}$.
Based on this insight, we propose to estimate the real $r_s$ in $\mathcal{D}_{y=0}$ by analyzing the proportion $\hat{r}_{s}$ in $\hat{\mathcal{D}}_{y=0}$. 
Thus, the challenge pertains to determining the presence of property $s$ in different synthetic samples. To address this challenge, we opt for a simple but effective solution, i.e., employ a property discriminator $g_{\text{d}}$.

%

\noindent
\textbf{The First Adversary $\mathcal{A}^{(1)}_{\text{PIA}}$}. 
If the receiver can collect an auxiliary dataset $\mathcal{D}_{\text{aux}}$ with binary labels indicating the existence or non-existence of the property $s$, a binary property discriminator $g_{\text{d}}$ can be trained effectively. 
Using $g_{\text{d}}$ to determine the presence of $s$ in each sampled image from $\hat{\mathcal{D}}_{y=0}$, we can obtain the distribution histogram of $S$ effectively. The detailed computation steps for $r_s$ are provided in Algorithm~\ref{alg-attack-PIA}.

\noindent
\textbf{The Second Adversary $\mathcal{A}^{(2)}_{\text{PIA}}$}.
When the auxiliary dataset is unavailable, manual labeling of $\dhat_{y=0}$ may be a viable option for property inference, which, however, can be costly and time-consuming~\cite{survey-SSLUL}. 
We have conducted extensive explorations to identify a suitable discriminator for various properties, with the pre-trained semantic model, CLIP~\cite{CLIP}, emerging as a promising candidate because of its multi-modal embedding space for both text and image as well as its excessive training visual concepts. One of the key advantages of CLIP is its ability to generalize effectively. 
Given its training with large-scale natural language supervision~\cite{CLIP}, it can be directly applied for discerning different properties with distinct semantic meanings without requiring additional fine-tuning efforts. In Fig.~\ref{fig-clip}, for instance, 
we can achieve zero-shot classification~\cite{zeroshotdiversity,zeroshoticlr,CLIP} of various properties by replacing the objects in the prompts.
In some cases, CLIP can achieve even better classification accuracies than discriminators trained on auxiliary datasets. For example, when discerning teenagers aged 18-20 from middle-aged individuals aged 30-39 on the AFAD dataset~\cite{AFAD}, CLIP achieves a 91\% accuracy, while the discriminator $g_{\text{d}}$ trained on AFAD attains only an 80\% accuracy.
Therefore, in situations where auxiliary datasets are unavailable, the adversary can utilize the pre-trained CLIP model as the discriminator $g_{\text{d}}$ for property inference.

\begin{algorithm}[!t]
\begin{small}
\caption{The Property Inference Attack}
\label{alg-attack-PIA}
\KwInput{the discriminator $g_{\text{d}}$, the sampled dataset $\dhat=\{\bx_i\}_{i=1}^{\hat{m}}$}
\KwOutput{The estimated property ratio $\hat{r}_s$}
$n_s\gets 0$ \tcp*{Initialize the Counter of $s$'s Presence}
\For{$i=1,2,\cdots,\hat{m}$}{  
    confidence $\gets g_{\text{d}}(\bx_i)$ \;
    $n_s \gets n_s + \vmathbb{1}(\text{confidence}>0.5)$ \tcp*{Update the Counter}
}
$\hat{r}_s\gets \frac{n_s}{\hat{m}}$ \;
\Return $\hat{r}_s$\;
\end{small}
\end{algorithm}

\begin{figure}[t]
\centering
\includegraphics[width=0.8\columnwidth]{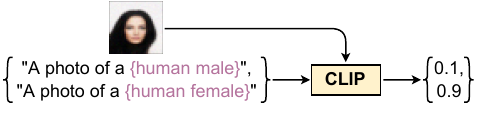}
\caption{Use the CLIP model as a property discriminator. }
\label{fig-clip}
\end{figure}

\subsection{Bounding the Estimation Error}\label{subsec-error-bound}
In Algorithm~\ref{alg-attack-PIA}, the receiver takes $\hat{m}$ samples from the diffusion model for estimating the property proportion of $s$.
By employing Hoeffding's inequality~\cite{hoeffding}, we can bound the estimation error as follows.

\begin{theorem}\label{thm-error-bound}
Let $g_{\text{d}}: \mathcal{X}\rightarrow \mathcal{S}$ used in Algorithm~\ref{alg-attack-PIA} be an unbiased discriminator\footnote{
\noindent Note that an unbiased discriminator can be trained using a variety of techniques, such as meta-learning~\cite{debial-metalearning}, conditional adversarial debiasing~\cite{debias-adversarial}, and re-weighted objective~\cite{debias-reweight}. Furthermore, in Section~\ref{subsec-exp-PIA}, we will discuss the influence of a biased discriminator on the accuracy of PIA.
} 
with a prediction error of $\epsilon_d$.
Let ${r}^*_s$ represent the ground truth proportion of $s$ in the synthetic data, and $\hat{r}_s$ denote the estimated proportion of $s$ produced by Algorithm~\ref{alg-attack-PIA}. There exists a constant $\epsilon\geq 0$ such that
\begin{equation}\label{eq-hoeffding-final}
    \prob\left(\left|\hat{r}_s - r^*_s \right|\geq \epsilon+\epsilon_d\right) \leq 2 \exp\left(-2\hat{m}\epsilon^2\right).
\end{equation}
\end{theorem}
\begin{proof}
Please refer to Appendix~\ref{appendix-proof-thm}.
\end{proof}

Eq.~\eqref{eq-hoeffding-final} demonstrates that by sampling $\hat{m}=\frac{\log {(2/\delta)}}{2\epsilon^2}$ images for some $\delta>0$, we can reduce the probability of $\hat{r}_s$ deviating from $r^*_s$ by a value greater than $\epsilon + \epsilon_d$ to less than $\delta$.
For example, if the prediction accuracy of $g_{\text{d}}$ is $90\%$ (i.e., $\epsilon_d=0.1$), a sample size of $\hat{m}=150$ images would suffice to ensure that the probability of $|\hat{r}_s-r^*_s|\geq 0.2$ is less than 0.1.

It is important to note that Eq.~\eqref{eq-hoeffding-final} only provides an error bound between the estimated proportion $\hat{r}_s$ and the ground truth proportion $r^*_s$ of the synthetic data distribution.
In order to accurately estimate the ground truth proportion $r_s$ in the sharer's training dataset, we also need to address the generative error of diffusion models, denoted as $|r^*_s - r_s|$. However, this generative error is contingent upon the generative capacity of diffusion models, which is intricately linked to the training hyper-parameters and model structures, and cannot be determined analytically~\cite{beatgan, NCSN}.
Nonetheless, in Section~\ref{subsec-exp-PIA}, we empirically validate the effectiveness of Eq.~\eqref{eq-hoeffding-final} across various diffusion models and real-world datasets. Our results indicate that, within the context of diffusion models, $|r^*_s - r_s|$ tends to be small ($<0.05$), and the estimated error of our PIA can be effectively bounded by Eq.~\eqref{eq-hoeffding-final}.

\subsection{Attack Generalization}
In scenarios where the property space $\mathcal{S}$ comprises multiple values, i.e., $\mathcal{S}=\{{s}_1,\cdots,{s}_k\}$ with $k>2$, the adversary can choose to train a multi-class discriminator or a one-vs-all ensemble model~\cite{onevsall}, i.e., training $k$ binary discriminators $\{g_{\text{d}}^{(i)}\}_{i=1}^k$ with $g_{\text{d}}^{(i)}$ determining the property $s={s}_i$. 
In this paper, we adopt the one-vs-all method such that Algorithm~\ref{alg-attack-PIA} and the error bound in Eq.~\eqref{eq-hoeffding-final} can be immediately generalized to the multi-class scenario without any modifications.

In addition, the proposed attack method can be easily generalized to estimate the overall proportion of property $s$ across different classes. Suppose there are $k$ classes in the training dataset $\dbig$, i.e., $\mathcal{Y}=\{0,\cdots,k-1\}$, and let $\{m_{y=i}\}_{i=0}^{k-1}$ denote the numbers of records with label $i$ in $\dbig$.
We can infer the overall property proportion $r_{s,\mathcal{Y}}$ by initially estimating the property proportion per class $\{r_{s,y=i}\}_{i=0}^{k-1}$ and then compute  
 $r_{s,\mathcal{Y}}=\frac{\sum_{i=0}^{k-1} m_{y=i}r_{s, y=i}}{\sum_{i=0}^{k-1} m_{y=i}}$.
Note that in many applications, the record numbers of each class in the training set are not considered sensitive and can be disclosed to third parties for model utility evaluation~\cite{gpt3}. Therefore, the overall property proportion can be easily inferred using the proposed attack.
However, it is important to emphasize that the fine-grained information, i.e., the property proportion of a specific class, can pose a greater risk than the overall property proportion in practice.
For example, knowledge of the property distribution among benign software instances in the training dataset of a malware detector can help an adversary evade a company's auditing mechanism, which is more valuable than knowing the overall property proportion across benign software and malware instances in the training records~\cite{propertyPermutation}.

\vspace{0.5mm}
\noindent
\textbf{Remarks on PIA}. 
In the context of $\mathcal{A}^{(2)}_{\text{PIA}}$, we conducted an extensive exploration of various semi-supervised learning (SSL) and unsupervised learning (UL) methods for training a property discriminator. However, none of them exhibited performance comparable to the CLIP model.
SSL and UL methods typically rely on the manifold assumption~\cite{survey-SSLUL} that images of the same class share similar structures.
In cases where all images in $\dhat$ exhibit identical patterns, such as human faces, UL and SSL are not effective enough for property discrimination.
For instance, even with the state-of-the-art SSL method FixMatch~\cite{SSLfixmatch}, achieving accurate discrimination of the male property on the CelebA dataset is still challenging in our experiments, with a maximum accuracy of only $65\%$, which is insufficient for reliable property inference.
Given CLIP's consistent and robust performance in distinguishing various properties (with accuracy $>90\%$) in our experiments, we choose to employ CLIP as the property discriminator for $\mathcal{A}^{(2)}_{\text{PIA}}$ and believe that developing a learning method from scratch for the setting of $\mathcal{A}^{(2)}_{\text{PIA}}$ is not that necessary.
Note that the CLIP model may not perform well on a few specialized tasks, such as medical analysis and flower species discrimination~\cite{CLIP}. Thus, when the target property $s$ is related to specific and uncommon concepts, the first adversary $\mathcal{A}^{(1)}_{\text{PIA}}$ is a more suitable approach for property inference.

\begin{table*}[!ht]
\setlength{\tabcolsep}{5pt}
\caption{The experimental datasets.}\label{tb-datasets}
\vspace{-1mm}
\center
\small
\begin{tabular}{cccc||cccc}
\hline
\multicolumn{4}{c||}{Fairness Poisoning Attack} & \multicolumn{4}{c}{Property Inference Attack} \\
\cline{1-8}
Dataset & Type & Sensitive Features $\mathcal{S}$ & Labels $\mathcal{Y}$ & Dataset  & Type & Target Property $s$ & Target Label $y$   \\
\hline 
\hline
CelebA-S & image &	$\{\text{male}, \text{female}\}$ &	$\{\text{smiling or not}\}$ &  MNIST & image&	the number ``$0$'' & number$<5$  \\                    
CelebA-A & image & $\{\text{male}, \text{female}\}$ &	$\{\text{attractive or not}\}$ & CelebA & image & the gender ``male'' & smiling  \\
 Adult-R & tabular & $\{\text{white}, \text{non-white}\}$ &	$\{\text{income$\geq 50000$ or not}\}$ & AFAD & image & the age group $[18, 20]$ & class ``$0$''   \\
 Adult-G & tabular & $\{\text{male}, \text{female}\}$ &	$\{\text{income$\geq 50000$ or not}\}$ & Adult & tabular &	the race ``white'' & income$\geq 50000$  \\
\hline
\end{tabular}
\end{table*}

\section{Experiments}\label{sec-exp}

\begin{figure}[t]
\centering
\includegraphics[width=0.99\columnwidth]{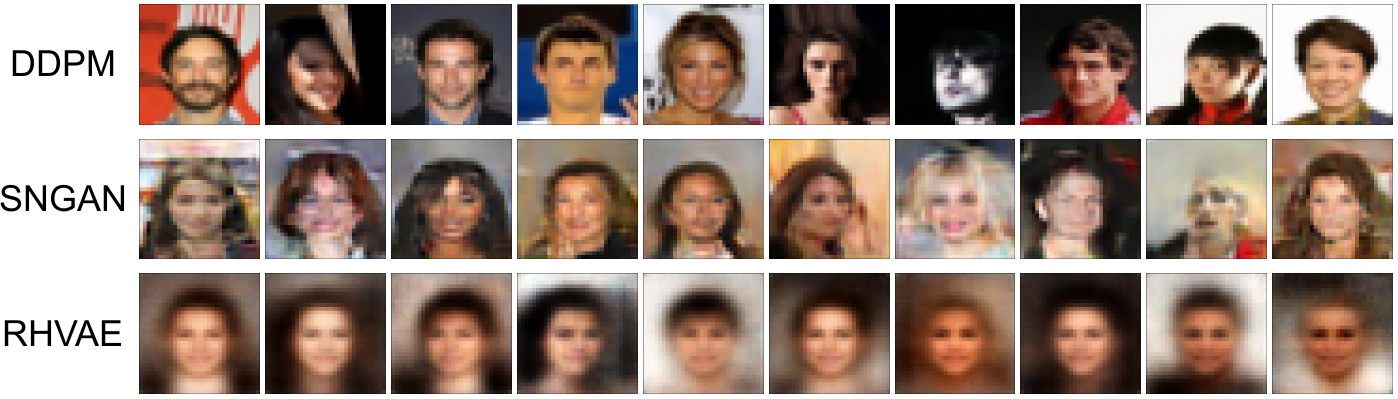}
\caption{Images sampled from different models. }
\label{fig-syn-imgs}
\end{figure}

\noindent
\textbf{Datasets and Models.} Following prior research on property inference and fairness poisoning~\cite{propertyPoison, propertyGAN, changbias}, we use three image datasets, i.e., \textit{MNIST}~\cite{mnist}, CelebFaces Attributes (\textit{CelebA})~\cite{celeba}, and the Asian Face Age Dataset (\textit{AFAD})~\cite{AFAD}, and one tabular dataset, i.e., \textit{Adult}~\cite{UCI}, in our experiments.
The property and label details of the experimental datasets are summarized in Table~\ref{tb-datasets}.
Note that for PIA, the inference errors across different properties are similar, and the overall property proportion is directly derived from the property proportions per class.
For ease of presentation, we only report the results of property inference on one target property within the target label for each dataset.
We employ three established diffusion models for image datasets, namely NCSN~\cite{NCSNv2}, DDPM~\cite{DDPM}, and SDEM~\cite{SDE}.
%
For the tabular dataset, we utilize TabDDPM~\cite{TabDDPM}.
Furthermore, we compare the attack performance on the diffusion models with that of two GAN models, i.e., the spectral normalization GAN (SNGAN)~\cite{SNGAN} and the self-supervised GAN (SSGAN)~\cite{SSGAN}, as well as one VAE model, i.e., Riemannian Hamiltonian VAE (RHVAE)~\cite{RHVAE}.
In the experiments, all models are trained from scratch until convergence. 
We give some synthetic samples in Fig.~\ref{fig-syn-imgs}.
More details about datasets and attack implementation are deferred to Appendix~\ref{appen-exp-setting}.


\noindent
\textbf{Evaluation Metrics.}
In fairness poisoning attacks, we evaluate both the losses of accuracy and fairness.
We use the ratio of accuracy degradation to evaluate the test accuracy loss caused by the poisoning attack:
\begin{equation}
    \ell_{\text{acc}} = \frac{\eta(f_{\text{c}},\dtest)-\eta(\fbias,\dtest)}{\eta(f_{\text{c}},\dtest)},
\end{equation}
where $\eta(f_{\text{c}},\dtest)$ denotes the accuracy of $f_{\text{c}}$ tested on $\dtest$, $f_{\text{c}}$ denotes the model trained on $\dhat\sim \prob_{\dbig}$, and $\fbias$ denotes the model trained on $\dhat \sim \prob_{\dpoi}$. 
If the proposed fairness poisoning attack preserves the data utility as expected, then $\ell_{\text{acc}}$ should be close to 0.
We expect $\ell_{\text{acc}}$ to be less than $5\%$ in the experiments.
To evaluate the fairness loss, we first follow previous studies~\cite{changbias, pf-chang} and define the fairness gap \textit{w.r.t.} demographic parity as $\Delta(f_{\text{c}},\dtest) = \max_{s,s'\in \mathcal{S}, (X,S)\subseteq \dtest}| \prob(f_{\text{c}}(X,S)= 1|S=s) - \prob(f_{\text{c}}(X,S)= 1|S=s')|$,
where $\Delta(f_{\text{c}},\dtest)$ should be close to 0 if $f_{\text{c}}$ is a fair model.
Then, we define the fairness loss as follows:
\begin{equation}
    \ell_{\text{fair}} = \Delta(\fbias,\dtest) - \Delta(f_{\text{c}},\dtest).
\end{equation}
A large $\ell_{\text{fair}}$ indicates a successful fairness poisoning attack.
For property inference attacks, we use the Mean Absolute Error ($\ell_1$ loss) to measure the errors between the estimated property ratio $\hat{r}_s$ and the ground truth ratio ${r}_s$:
\begin{equation}
  \ell_1 (\hat{r}_s, r_s) =|\hat{r}_s-r_s|.
\end{equation}
Since $r_s$ is bounded between 0 and 1 among different settings, the $\ell_1$ loss suffices to compare the attack errors reasonably.
In addition, unless explicitly specified, we conduct 10 independent attacks with different random seeds for each experiment and report the averaged results.

\noindent
\textbf{Baselines.}
For the fairness poisoning attack, we adopt the label flipping method as the baseline, commonly used in previous poisoning studies~\cite{pf-chang, pf-labelflip, fp-accuracy}. However, to enhance its strength, we opt for maximizing $I(X_p,S_p;Y_p)$ as described in Eq.~\eqref{eq-opt-sim} to flip the labels of the $\alpha-$fraction poisoning samples, instead of randomly flipping them as in~\cite{pf-chang}.
For the property inference attack, it is worth noting that existing methods in the literature~\cite{propertyFL, propertyMembership, propertyPermutation, propertyPoison} have mainly focused on inferring the presence of the target property in white-box settings, which significantly differ from our attacks. To provide a reasonable comparison, we employ a random guess baseline~\cite{luo2021feature, luoccs}, which randomly samples a value from the Uniform distribution $\boldsymbol{U}(0,1)$ as the estimation of $r_s$.

\begin{figure*}[t]
\centering
\begin{small}
\begin{tabular}{cccc}
\multicolumn{2}{r}{\hspace{0mm} \includegraphics[height=4mm]{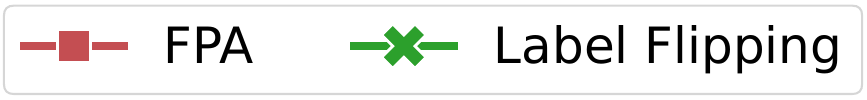}}
&
\multicolumn{2}{l}{\hspace{0mm} \includegraphics[height=4mm]{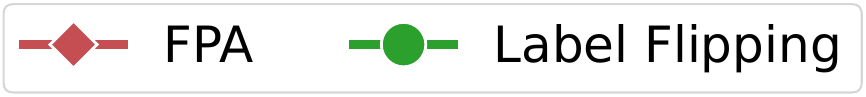}}
\vspace{-4mm}  \\
\hspace{-4mm}
\subfloat[CelebA-A]{\includegraphics[width=0.25\textwidth]{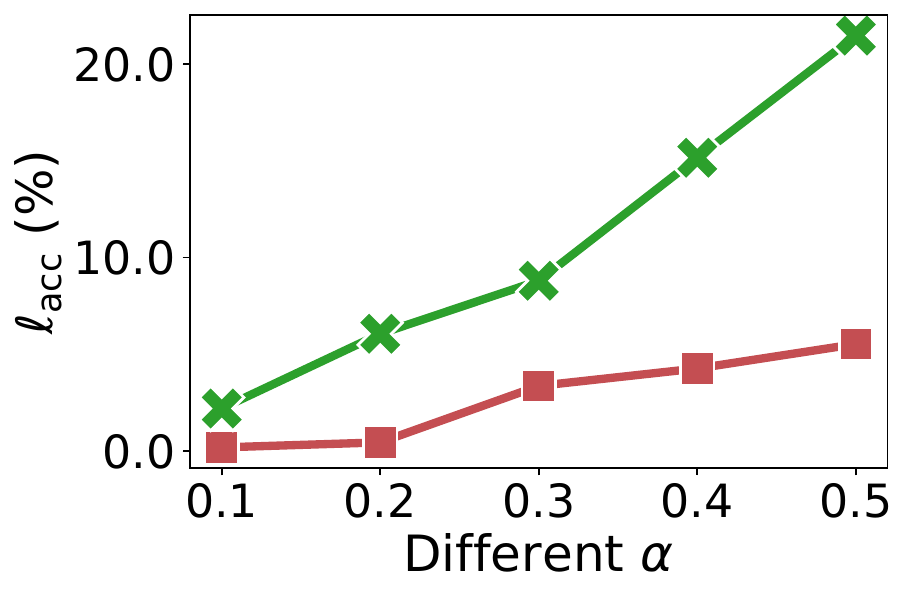}\label{subfig-FPA-acc-diff-alpha-celeba-a}}
&
\hspace{-5mm}
\subfloat[CelebA-S]{\includegraphics[width=0.25\textwidth]{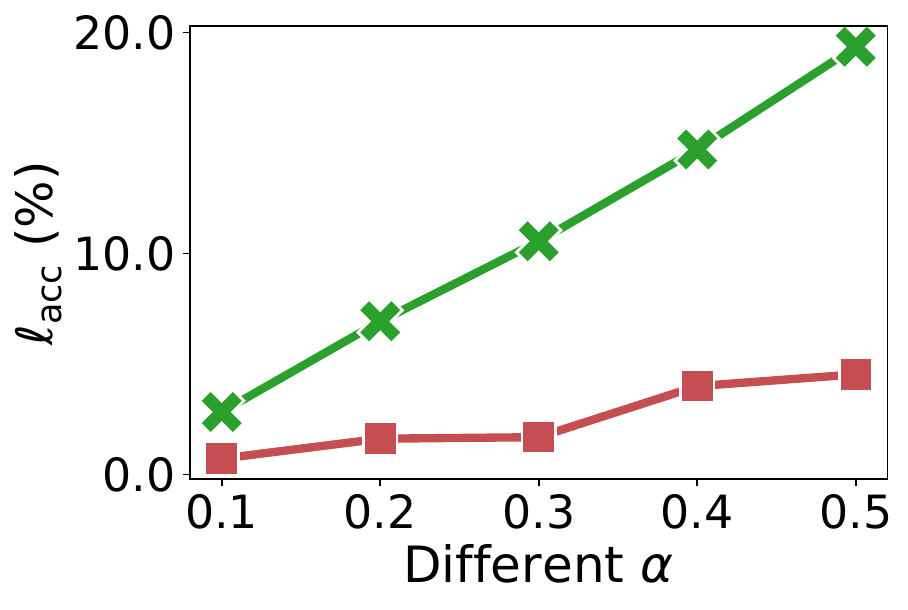}}
&
\hspace{-5mm}
\subfloat[Adult-G]{\includegraphics[width=0.25\textwidth]{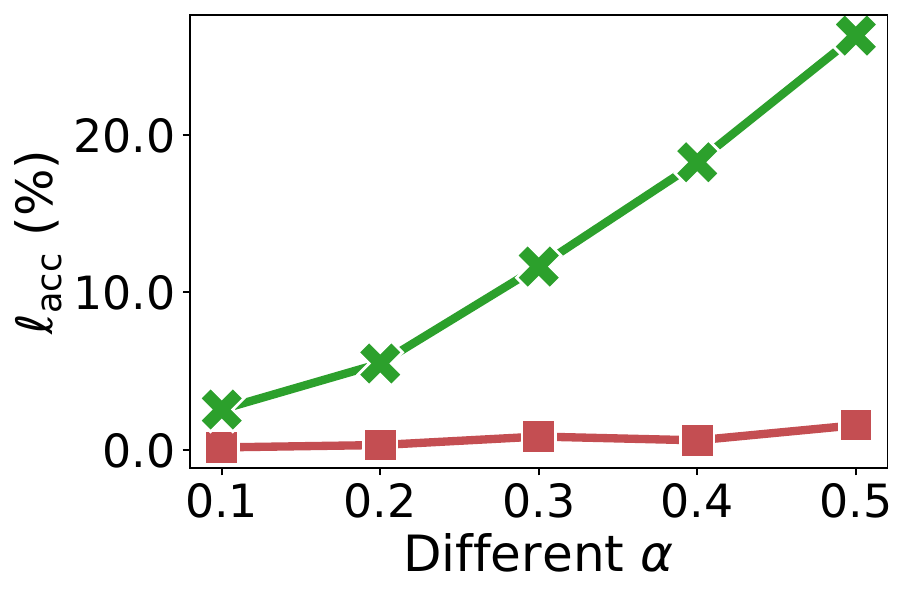}}
&
\hspace{-5mm}
\subfloat[Adult-R]{\includegraphics[width=0.25\textwidth]{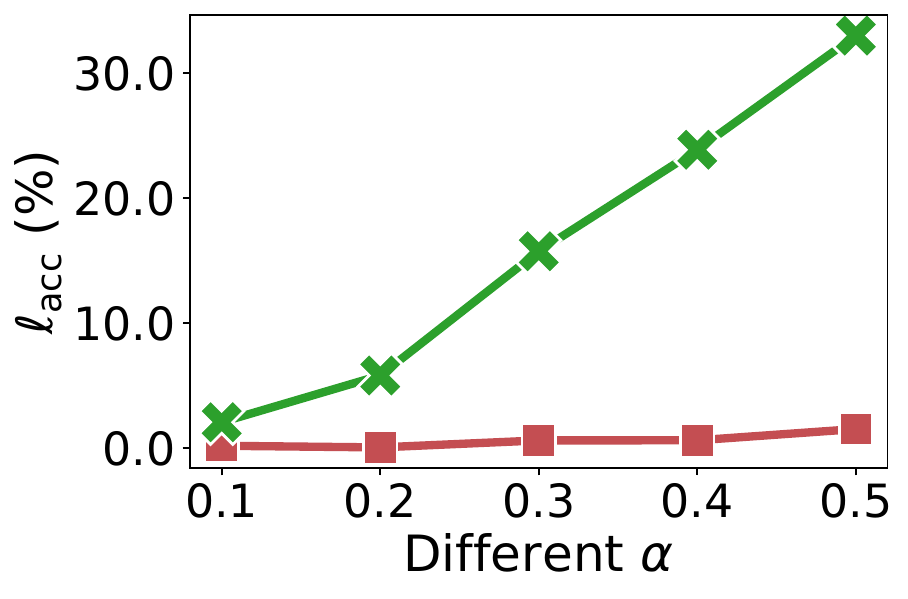}\label{subfig-FPA-acc-diff-alpha-adult-r}}
\vspace{-2mm}  \\
\hspace{-4mm}
\subfloat[CelebA-A]{\includegraphics[width=0.25\textwidth]{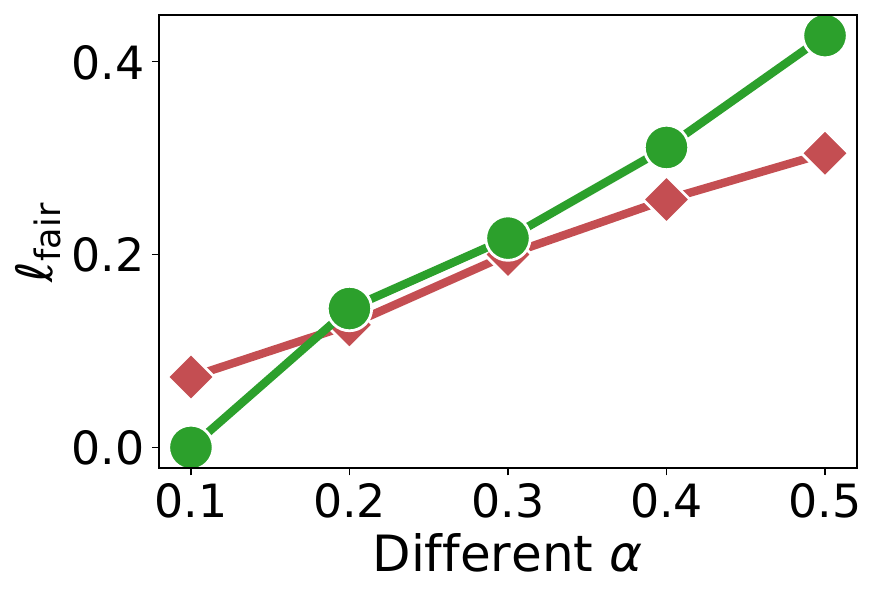}\label{subfig-FPA-fair-diff-alpha-celeba-a}}
&
\hspace{-5mm}
\subfloat[CelebA-S]{\includegraphics[width=0.25\textwidth]{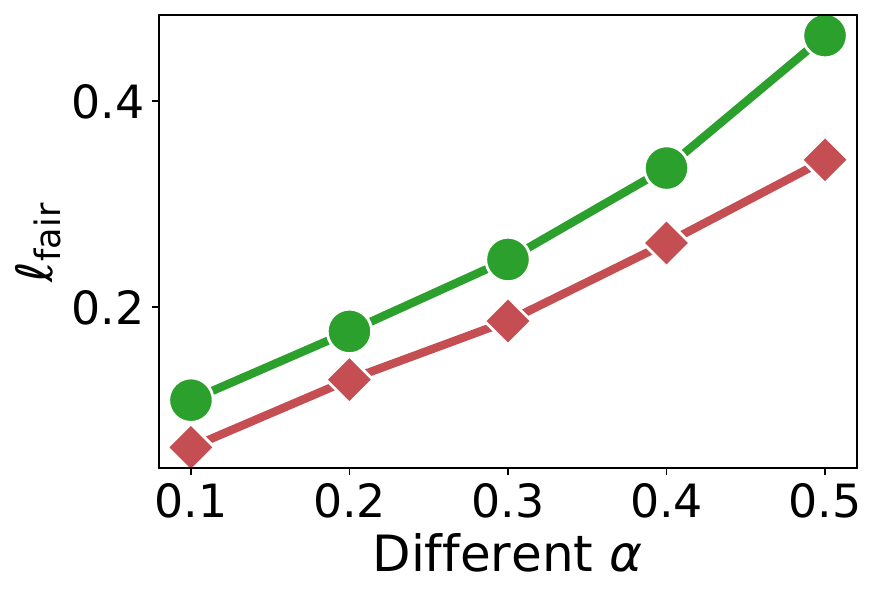}}
&
\hspace{-5mm}
\subfloat[Adult-G]{\includegraphics[width=0.25\textwidth]{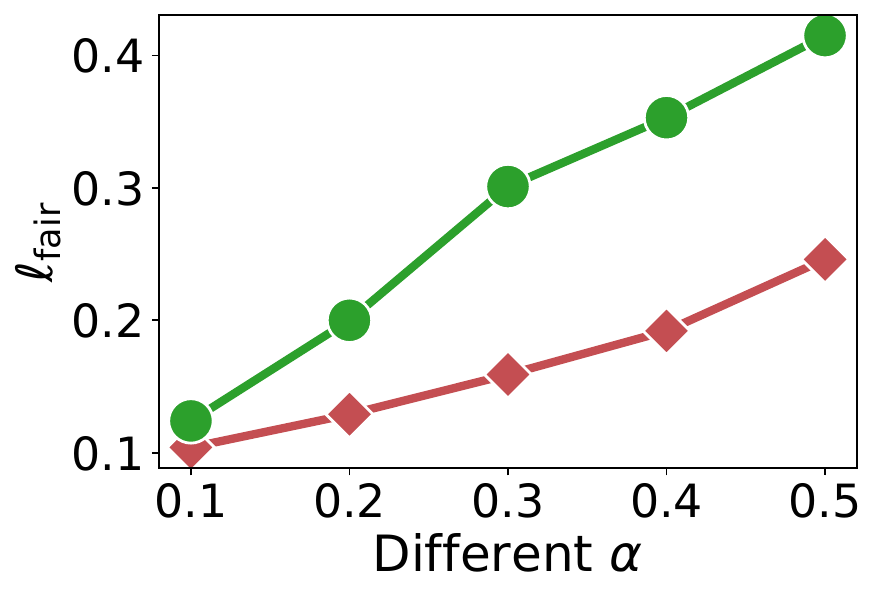}\label{subfig-FPA-fair-diff-alpha-adult-g}}
&
\hspace{-5mm}
\subfloat[Adult-R]{\includegraphics[width=0.25\textwidth]{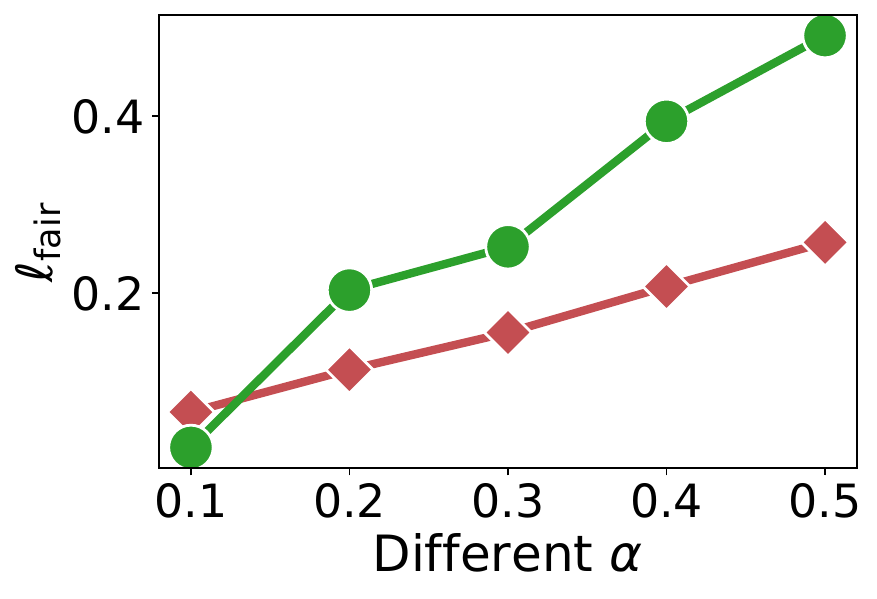}\label{subfig-FPA-fair-diff-alpha-adult-r}}
\end{tabular}
\caption{The (a)-(d) accuracy loss and (e)-(h) fairness loss caused by fairness poisoning attacks on different datasets.}
\label{fig-FPA-diff-alpha}
\end{small}
\end{figure*}

\subsection{Evaluation on Fairness Poisoning Attacks}
In this subsection, we evaluate the proposed fairness poisoning attack on two datasets: CelebA and Adult.
To obtain a comprehensive understanding of the performance of FPA, we divide CelebA and Adult datasets into sub-datasets based on labels and sensitive features, respectively. 
Specifically, we create two sub-datasets from CelebA based on the labels determining whether a face is attractive (CelebA-A) or smiling (CelebA-S), and two sub-datasets from Adult based on the sensitive features Gender $\mathcal{S}=\{\text{male}, \text{female}\}$ (Adult-G) and Race $\mathcal{S}=\{\text{white}, \text{nonwhite}\}$ (Adult-R).
By performing experiments on these four new datasets, we can examine the impact of FPA on different downstream tasks and evaluate its robustness in the face of different sensitive features.
In addition, for the clarity of presentation, we give attack results in the context of zero-shot learning in this section. Results pertaining to few-shot learning scenarios are deferred to Appendix~\ref{appendix-few-shot}.

\noindent
{\textbf{Attack Performance \textit{w.r.t.} Different Poisoning Proportions $\alpha$.}
Note that the threshold $\xi$ in the data utility constraint of Eq.~\eqref{eq-opt-sim} has a similar impact on attack results as the poisoning ratio $\alpha$. Hence, we fix $\xi=0.1c$ in Eq.~\eqref{eq-opt-sim} and vary $\alpha\in\{0.1, 0.2, 0.3, 0.4, 0.5\}$ in Algorithm~\ref{alg-FPA-greedy} for performance comparison.
%
%
Different DDPM models and downstream classifiers are trained accordingly to evaluate the accuracy and fairness losses. The results are presented in Fig.~\ref{fig-FPA-diff-alpha}.
There are four main observations.
First, the accuracy losses caused by FPA are generally less than $5\%$ for different $\alpha$, as shown in Fig.~\ref{subfig-FPA-acc-diff-alpha-celeba-a} - \ref{subfig-FPA-acc-diff-alpha-adult-r}. This observation confirms the effectiveness of the data utility constraint in the optimization objective (Eq.~\eqref{eq-opt}).
Second, the fairness losses increase with the increase of $\alpha$, as shown in Fig.~\ref{subfig-FPA-fair-diff-alpha-celeba-a} - \ref{subfig-FPA-fair-diff-alpha-adult-r}, which is expected since more poisoning data in the training dataset $\dpoi$ will inject a larger bias in the downstream classifier. 
Third, the label flipping method can cause more fairness loss than the proposed FPA, because the former aims to solve an unconstrained problem while the latter is constrained.
Correspondingly, the accuracy loss caused by label flipping is also much larger than that caused by FPA. The reason is that label flipping directly changes the correlation between inputs and labels in $\dpoi$, which could become significantly different from the distribution of $\dtest$, leading to lower test accuracy.
This observation validates the effectiveness of the proposed greedy algorithm (Algorithm~\ref{alg-FPA-greedy}) for solving the optimization objective in Eq.~\eqref{eq-opt}.
Fourth, in some cases with $\alpha=0.1$, the fairness loss caused by label flipping is smaller than the loss caused by FPA, e.g., Fig.~\ref{subfig-FPA-fair-diff-alpha-celeba-a} and \ref{subfig-FPA-fair-diff-alpha-adult-r}. 
The reason is that at a low $\alpha$, label flipping can be used as a regularization technique to help improve the model's prediction accuracy~\cite{dataaug-survey}.
In summary, the results 
demonstrate that Algorithm~\ref{alg-FPA-greedy} can significantly poison the fairness of downstream models meanwhile preserving their prediction accuracy.

\begin{table*}[htbp]
\setlength{\tabcolsep}{4pt}
\caption{Performance comparison of the proposed FPA \textit{w.r.t.} different models. LF indicates the label flipping baseline.}\label{tb-comp-FPA-model}
\vspace{-1mm}
\center
\small
\begin{tabular}{c|c|cc|cc|cc||cc|cc||cc}
\hline
\multirow{2}{*}{Dataset} &\multirow{2}{*}{Loss} & \multicolumn{2}{c|}{DDPM} & \multicolumn{2}{c|}{NCSN} & \multicolumn{2}{c||}{SDEM} & \multicolumn{2}{c|}{SNGAN} & \multicolumn{2}{c||}{SSGAN} & \multicolumn{2}{c}{RHVAE} \\
\cline{3-14}
  &   & FPA  & LF & FPA & LF & FPA &
   LF  & FPA  & LF & FPA & LF  & FPA & LF \\
\hline 
\hline
\multirow{2}{*}{CelebA-A}   & $\ell_{\text{acc}}$ & 2.61\%   & 19.04\%& 1.65\%  & 19.04\%  &	2.82\%&	17.76\% & 7.15\% & 16.85\% & 6.49\% & 14.49\% & 6.59\%	& 17.08\%
\\
                            & $\ell_{\text{fair}}$ & 0.2955  & 0.4177 & 0.2781  & 0.3734   & 	0.2664	& 0.3633 & 0.2532 & 0.3496 &  0.1453 &   0.3372 & 0.2511 & 0.3299  
                            \\                      
\hline
\multirow{2}{*}{CelebA-S}    & $\ell_{\text{acc}}$ & 3.74\%  & 19.35\%& 3.21\%  & 29.20\% &		2.66\%&	27.61\% & 4.28\% & 28.24\% & 5.49\% &  23.12\% & 7.56\%	& 20.85\%
\\
                            & $\ell_{\text{fair}}$ & 0.3325	 & 0.4542 & 0.3124  & 0.4208   & 	0.2888	&0.3784 & 0.2365 & 0.2626 &  0.1715 &   0.3687 & 0.2191 & 0.3560  
                            \\                       
\hline
\end{tabular}
\end{table*}

\begin{figure*}[t]
\centering
\begin{small}
\begin{tabular}{ccccc}
\hspace{-4mm}
\subfloat[{$\dbig$}]{\includegraphics[width=0.2\textwidth]{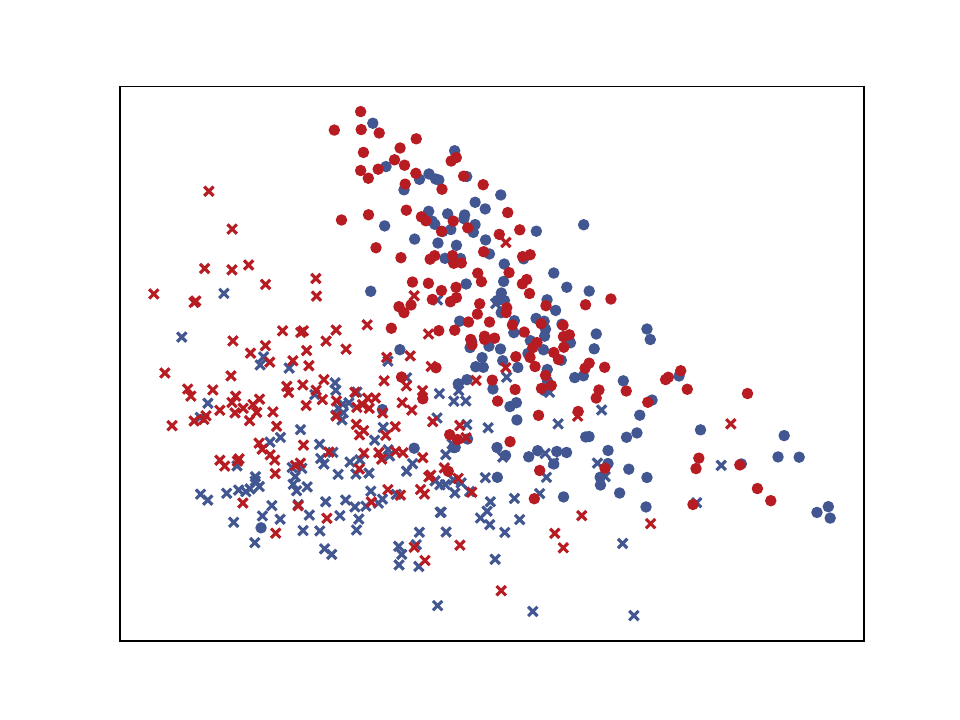}\label{subfig-dist-orig}}
&
\hspace{-4mm}
\subfloat[{$\dpoi$}]{\includegraphics[width=0.2\textwidth]{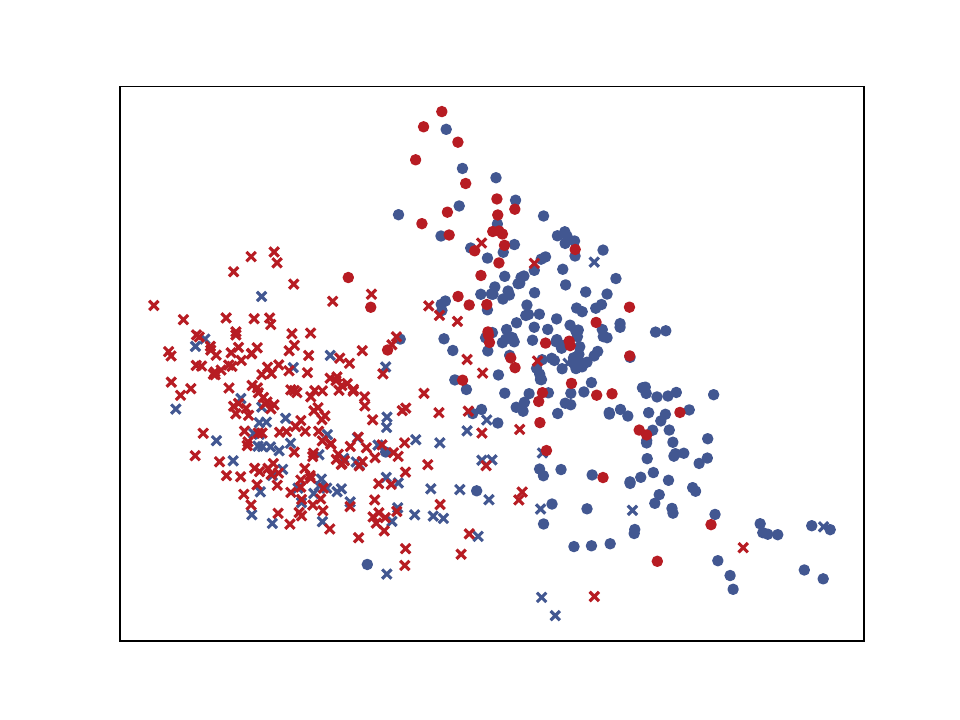}\label{subfig-dist-poi}}
&
\hspace{-4mm}
\subfloat[{{$\dhat$ from NCSN}}]{\includegraphics[width=0.2\textwidth]{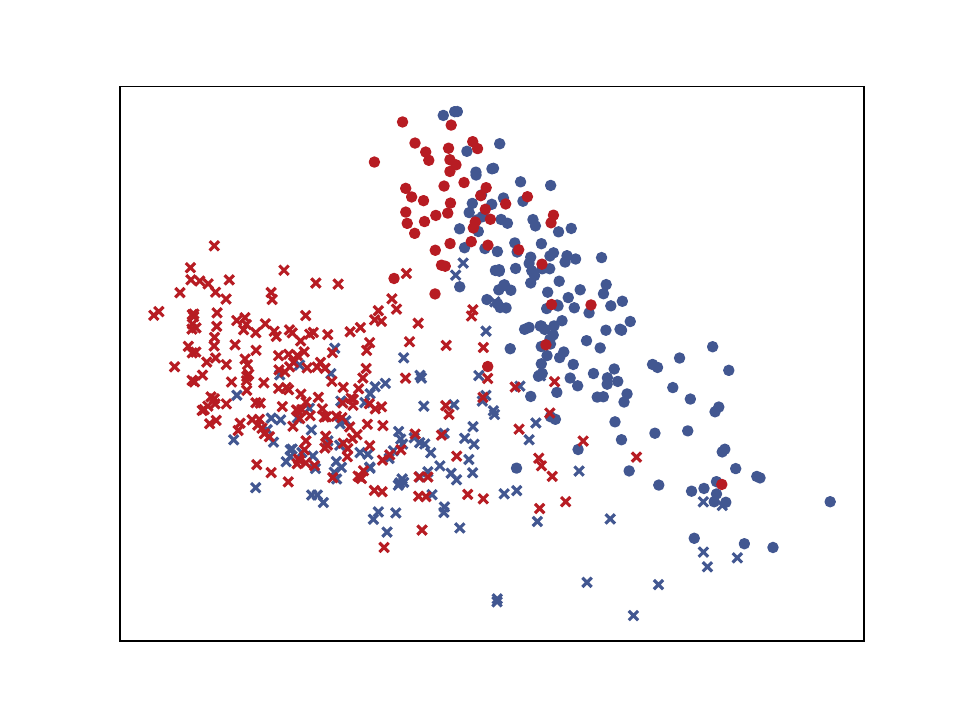}\label{subfig-dist-ncsn}}
&
\hspace{-4mm}
\subfloat[{{$\dhat$ from SNGAN}}]{\includegraphics[width=0.2\textwidth]{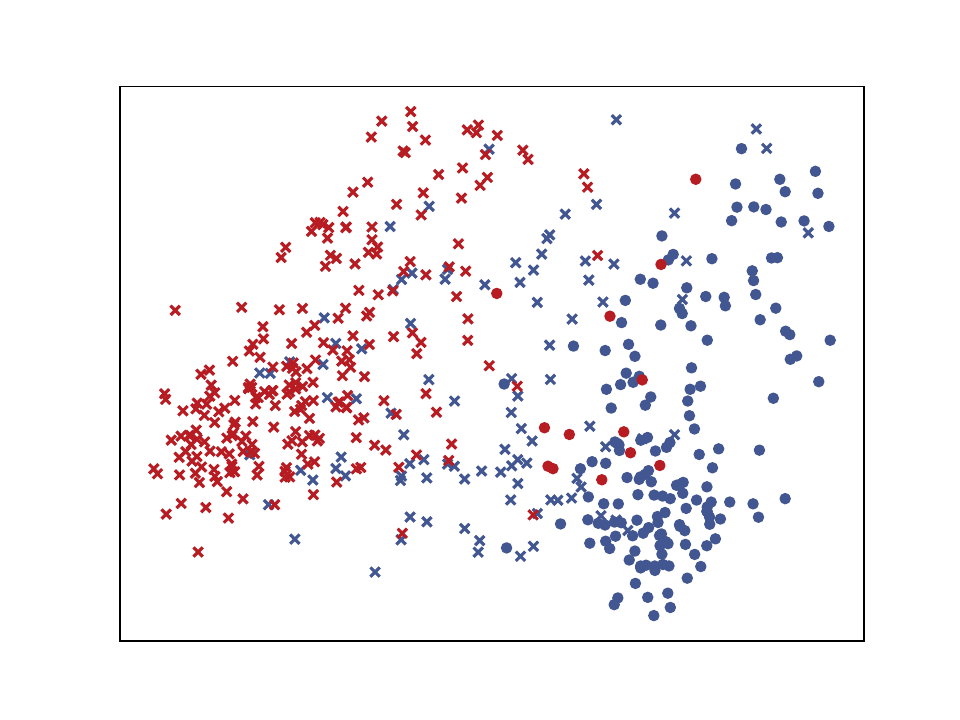}\label{subfig-dist-sngan}} 
&
\hspace{-4mm}
\subfloat[{{$\dhat$ from RHVAE}}]{\includegraphics[width=0.2\textwidth]{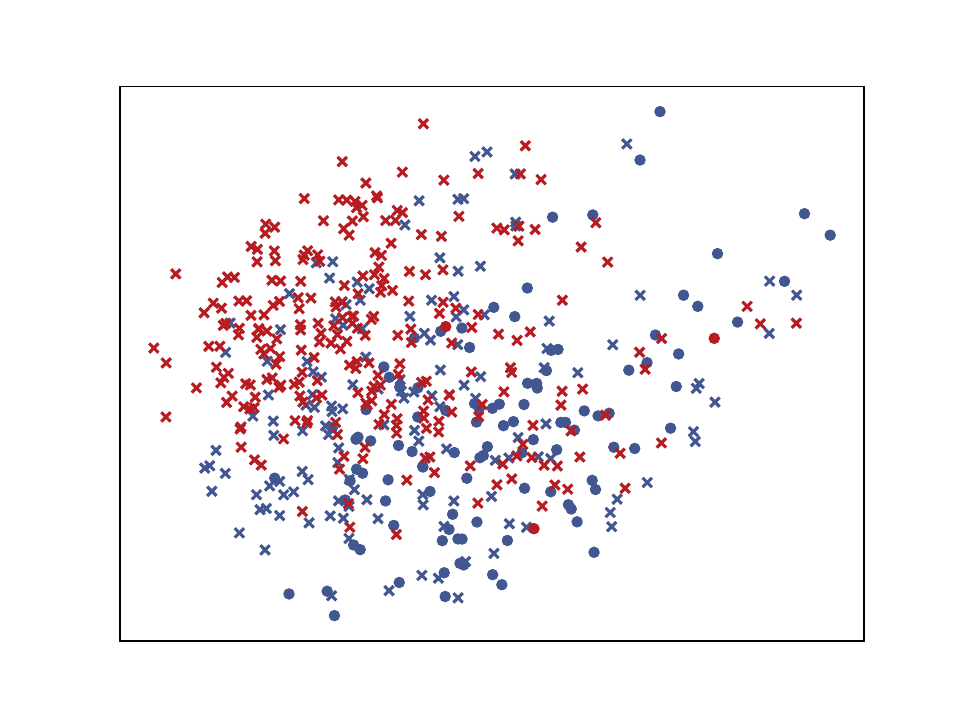}\label{subfig-dist-rhvae}}
\end{tabular}
\caption{Examples of the distributions of $\dbig$, $\dpoi$, and $\dhat$ based on CelebA-S. Different colors denote different classes. $\times$ denotes male, and $\circ$ denote female. Please zoom in for better viewing.}
\label{fig-dist}
\end{small}
\end{figure*}

\noindent
{\textbf{Attack Performance \textit{w.r.t.} Different Generative Models.}
In this subsection, we present a comparative analysis of the effectiveness of fairness poisoning attacks when tested on different generative models. The experimental evaluation is performed on the CelebA-A and CelebA-S datasets.
We start by fixing the poisoning proportion $\alpha$ to 0.5 in the proposed FPA and then assess the accuracy and fairness losses on different generative models. Table~\ref{tb-comp-FPA-model} summarizes the attack results.
The analysis of the results leads to three main observations.
First, the accuracy losses of FPA tested on diffusion models are smaller than those tested on other models. Specifically, the $\ell_{\text{acc}}$ of FPA tested on DDPM and NCSN are $2.61\%$ and $1.65\%$, respectively, which are significantly smaller than the $\ell_{\text{acc}}$ tested on SNGAN and RHVAE, i.e., $7.15\%$ and $6.59\%$.
This trend can be attributed to that the datasets $\dhat$ sampled from GAN and VAE models have relatively worse visual quality, which results in a slight distribution drift between $\dhat$ and $\dtest$, hence worse test accuracy.
Second, the fairness losses of FPA tested on diffusion models are larger than those tested on other models. For instance, the $\ell_{\text{fair}}$ of FPA tested on DDPM and NCSN are $0.2955$ and $0.2781$, which are larger than the $\ell_{\text{fair}}$ tested on SNGAN and RHVAE, i.e., $0.2532$ and $0.2511$.
This observation can be attributed to the relatively inferior distribution coverage exhibited by GANs, as well as the VAE's limitations in capturing fine-grained visual details (see Fig.~\ref{fig-syn-imgs}), which reduces the bias transmitted from $\dpoi$ to $\dhat$ along the workflow in Fig.~\ref{fig-overview}.
Third, we note that the accuracy losses of label flipping (LF) tested on diffusion models are larger than those tested on GANs. For example, the $\ell_{\text{acc}}$ of LF tested on DDPM and NCSN are $19.04\%$, which are higher than the $\ell_{\text{acc}}$ tested on SNGAN and SSGAN, i.e., $16.85\%$ and $14.49\%$.
This observation can also be attributed to diffusion models' better distribution coverage than GANs. A larger distribution bias transmitted from $\dpoi$ to $\dhat$ will cause the classifier to learn a different distribution from that of $\dtest$, resulting in worse test accuracy.

To further illustrate the impact of the proposed FPA on the workflow in Fig.~\ref{fig-overview}, we depict the distribution of $\dbig$, $\dpoi$, and three $\dhat$ generated from NCSN, SNGAN, and RHVAE in Fig.~\ref{fig-dist}. The points in these figures are produced by first extracting image semantic features via the CLIP model and then utilizing principal component analysis~\cite{PCA} to extract the two most important features as the coordinates of different images.
%
%
Comparing Fig.~\ref{subfig-dist-orig} with Fig.~\ref{subfig-dist-poi}, the distributions of blue points and red points in $\dpoi$ are similar to those in $\dbig$, indicating that the sampled poisoning dataset $\dpoi$ well preserves the data utility of $\dbig$.
Moreover, we see from Fig.~\ref{subfig-dist-poi} that the $\circ$ points in red color and the $\times$ points in blue color are significantly reduced compared to the original dataset (Fig.~\ref{subfig-dist-orig}), indicating the fairness bias injected by Algorithm~\ref{alg-FPA-greedy}.
Comparing Fig.~\ref{subfig-dist-ncsn} with Fig.~\ref{subfig-dist-sngan} and \ref{subfig-dist-rhvae}, we observe that the samples generated from NCSN can well preserve the distribution of $\dpoi$, while the samples from SNGAN and RHVAE follow shifted distributions that are different from $\prob_{\dpoi}$.

In summary, the proposed FPA can achieve relatively better attack results on diffusion models than on GANs and VAEs, i.e., smaller accuracy losses and larger fairness losses.



\begin{figure*}[!ht]
\centering
\begin{small}
\begin{tabular}{cccc}
\hspace{-4mm}
\subfloat[{MNIST $\mathcal{A}^{(1)}_{\text{PIA}}$}]{\includegraphics[width=0.25\textwidth]{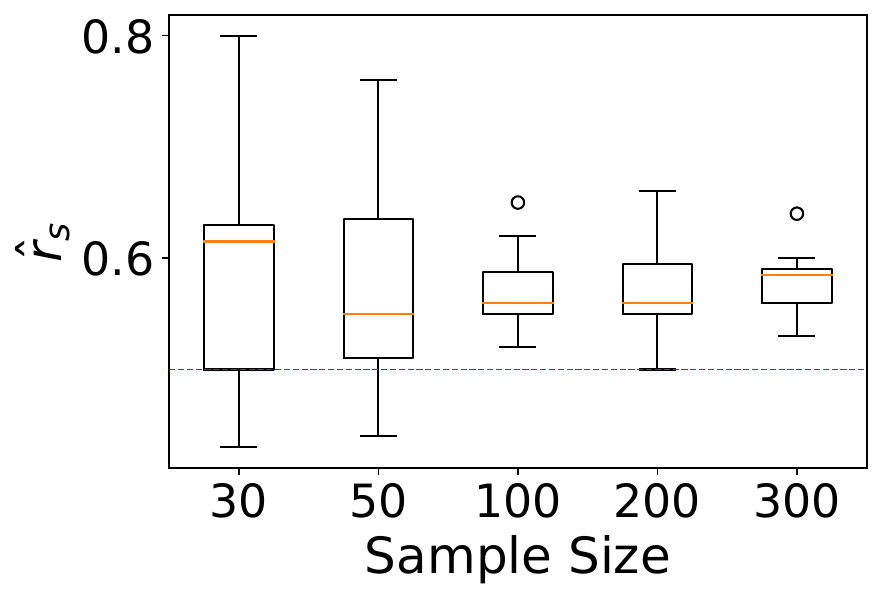}\label{subfig-box-mnist-withaux}}
&
\hspace{-5mm}
\subfloat[{MNIST $\mathcal{A}^{(2)}_{\text{PIA}}$}]{\includegraphics[width=0.25\linewidth]{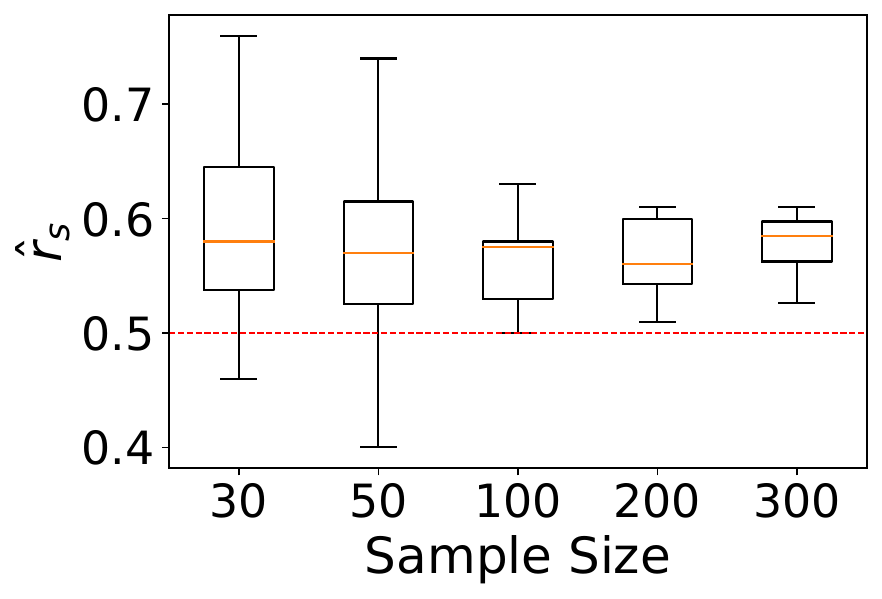}}
&
\hspace{-5mm}
\subfloat[{{CelebA $\mathcal{A}^{(1)}_{\text{PIA}}$}}]{\includegraphics[width=0.25\linewidth]{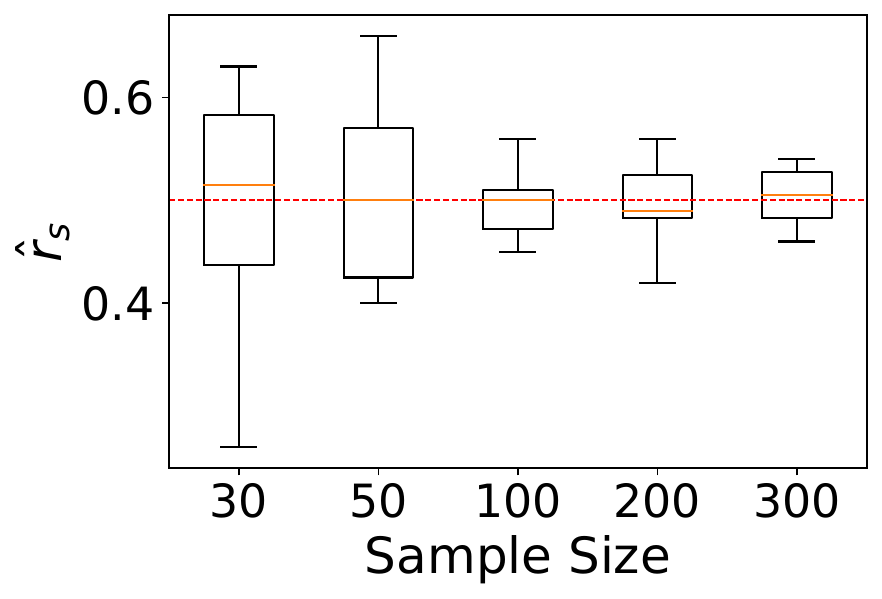}\label{subfig-box-celeba-withaux}}
&
\hspace{-5mm}
\subfloat[{{CelebA $\mathcal{A}^{(2)}_{\text{PIA}}$}}]{\includegraphics[width=0.25\linewidth]{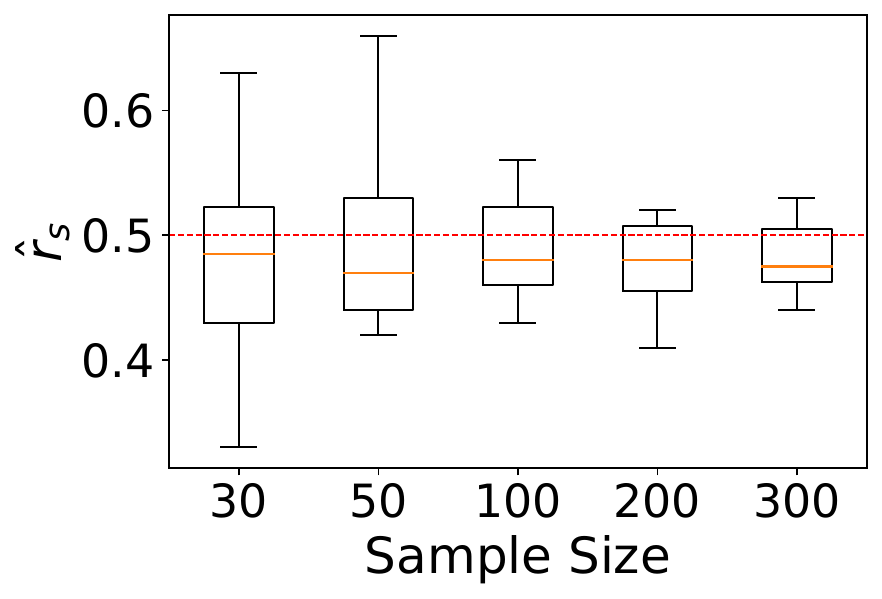}\label{subfig-box-celeba-noaux}} 
\end{tabular}
\centering
\caption{The results of ten PIA attacks on MNIST and CelebA. The red lines denote the ground truth $r_s=0.5$.}
\label{fig-box-mnist-celeba}
\end{small}
\end{figure*}

\begin{figure}[!ht]
\centering
\begin{small}
\begin{tabular}{cc}
\hspace{-4mm}
\subfloat[{AFAD $\mathcal{A}^{(1)}_{\text{PIA}}$}]{\includegraphics[width=0.25\textwidth]{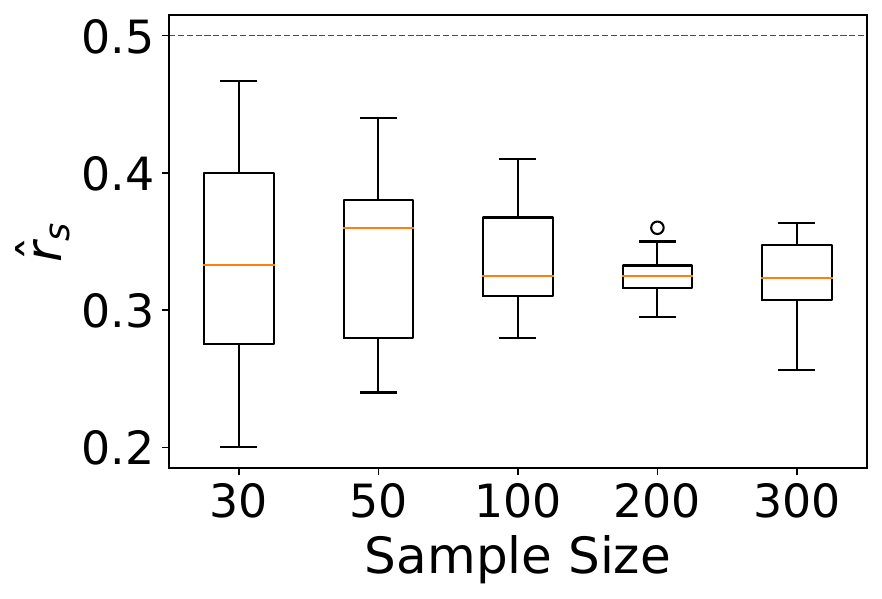}\label{subfig-box-afad-withaux}}
&
\hspace{-5mm}
\subfloat[{AFAD $\mathcal{A}^{(2)}_{\text{PIA}}$}]{\includegraphics[width=0.25\textwidth]{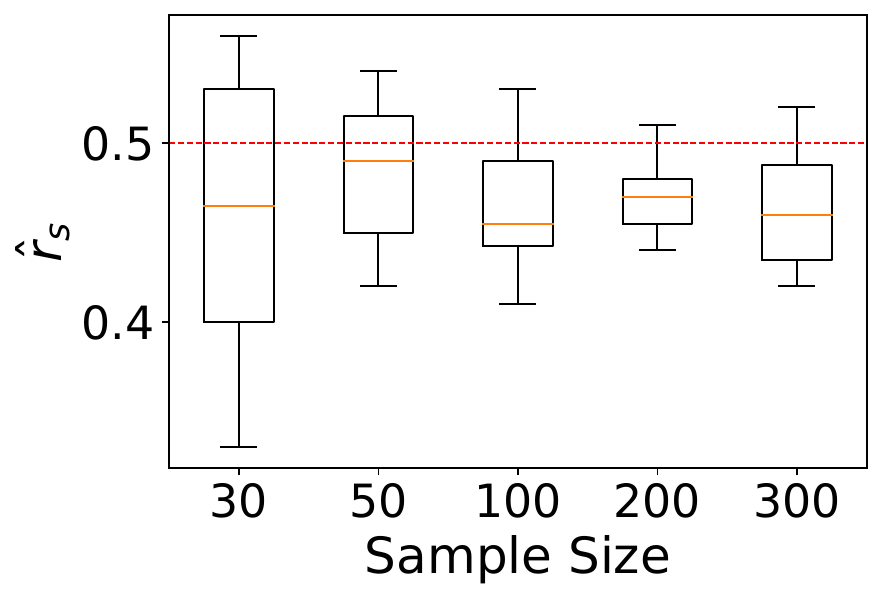}\label{subfig-box-afad-noaux}}
\vspace{-2mm}
\end{tabular}
\caption{The results of ten PIA attacks on AFAD with the red lines denoting the ground truth $r_s=0.5$.}
\label{fig-box-afad}
\end{small}
\end{figure}

\subsection{Evaluation on Property Inference Attacks}\label{subsec-exp-PIA}
\noindent
{\textbf{Evaluation on the Estimation Error of PIA.} In this subsection, we empirically evaluate the error bound of PIA described in Eq.~\eqref{eq-hoeffding-final}. 
Based on Eq.~\eqref{eq-hoeffding-final}, if fixing $\epsilon=0.1$ and sampling $\{30, 50, 100, 200, 300\}$ images, we can ensure that the probabilities of $|\hat{r}_s-r_s|\geq 0.1+\epsilon_d$ are less than $\{1, 0.736, 0.271, 0.037,0.005\}$, where $\epsilon_d$ denotes the prediction error of the property discriminator.
Now we fix $r_s=0.5$ and train NCSN models on different datasets. Then, we perform PIA on these trained models ten times and visualize the resulting estimations $\hat{r}_s$ using box plots, as shown in Fig.~\ref{fig-box-mnist-celeba} and \ref{fig-box-afad}. 
In $\mathcal{A}^{(1)}_{\text{PIA}}$, the $\epsilon_d$s of property discriminators trained on MNIST, CelebA, and AFAD are 0.04, 0.08, and 0.2, respectively.
From Fig.~\ref{fig-box-mnist-celeba} and \ref{fig-box-afad}, we make three main observations.
First, increasing the sample size results in a convergence of the estimated $\hat{r}_s$ to a smaller value range that is closer to the ground truth $r_s$, as illustrated in Fig.~\ref{subfig-box-celeba-withaux} and \ref{subfig-box-celeba-noaux}.
This observation validates the estimation error bound in Eq.~\eqref{eq-hoeffding-final}.
Second, the attack results of $\mathcal{A}^{(2)}_{\text{PIA}}$ are comparable to those of $\mathcal{A}^{(1)}_{\text{PIA}}$ in MNIST and CelebA.
This indicates that if auxiliary data is unavailable, using the pre-trained CLIP model to perform PIA can still achieve impressive performance.
Third, the estimated error is impacted by the accuracy of the property discriminator.
For instance, in the case where the property discriminator is trained on AFAD with an $\epsilon_d$ of 0.2, a larger estimation error is observed in Fig.~\ref{subfig-box-afad-withaux} compared to the attacks on other datasets.
In this scenario, using a pre-trained CLIP model with an $\epsilon_d$ of 0.1 can improve the results as shown in Fig.~\ref{subfig-box-afad-noaux}.

\noindent
{\textbf{Attack Performance \textit{w.r.t.} Different Property Proportions $r_s$.}
In this subsection, we evaluate the performance of the proposed PIA \textit{w.r.t.} different ground truth $r_s$.
Specifically, we fix the sample size to 200 and use NCSN and TabDDPM as the diffusion models for image datasets and tabular datasets, respectively.
We then vary the target property proportion $r_s$ among $\{0.1, 0.3, 0.5, 0.7, 0.9\}$ in $\dbig$ and train multiple diffusion models. Subsequently, we apply the proposed PIA on these diffusion models and report the attack results in Fig.~\ref{fig-PIA-diff-ratio}.
Note that for the tabular dataset Adult, the target property $s$ is a numerical feature that is readily obtainable from the samples. Therefore, there is no need to train a property discriminator specifically for Adult. We directly give the attack results on Adult as $\{0.064,	0.039,	0.040,	0.011,	0.043\}$.
As depicted in Fig.~\ref{fig-PIA-diff-ratio}, our proposed PIA approach can accurately estimate different target property proportions $r_s$, typically with errors less than 0.1.
This observation indicates the effectiveness and robustness of the proposed PIA approach.
Moreover, in Fig.~\ref{subfig-PIA-diff-ratio-withaux}, we notice that with the increase of $r_s$, the attack errors on AFAD also increase, which is due to the \textit{bias} of the property discriminator $g_{\text{d}}$, specifically, the prediction error $\epsilon_d=0.2$ for the target property age $[18, 20]$ and $\epsilon_d=0.1$ for age $[30, 39]$. 
To illustrate the impact of this bias on performance degradation, consider a scenario where 100 images are sampled for performing PIA via the biased $g_{\text{d}}$.
If the ground truth $r_s=0.1$, then the estimated $\hat{r}_s$ is 0.17 with an error of 0.07; meanwhile, when $r_s=0.9$, the estimated $r_s$ becomes 0.73, incurring a larger error of 0.17. 
This finding indicates that to achieve robust and reliable PIA, the attacker needs to carefully ensure the unbiasedness of $g_{\text{d}}$ discussed in Section~\ref{subsec-error-bound}.


\begin{figure}[t]
\centering
\begin{small}
\begin{tabular}{cc}
\multicolumn{2}{c}{\hspace{0mm} \includegraphics[height=4mm]{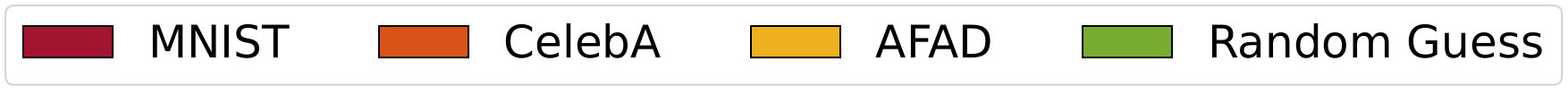}}
\vspace{-4mm}  \\
\hspace{-4mm}
\subfloat[$\mathcal{A}^{(1)}_{\text{PIA}}$]{\includegraphics[width=0.48\columnwidth]{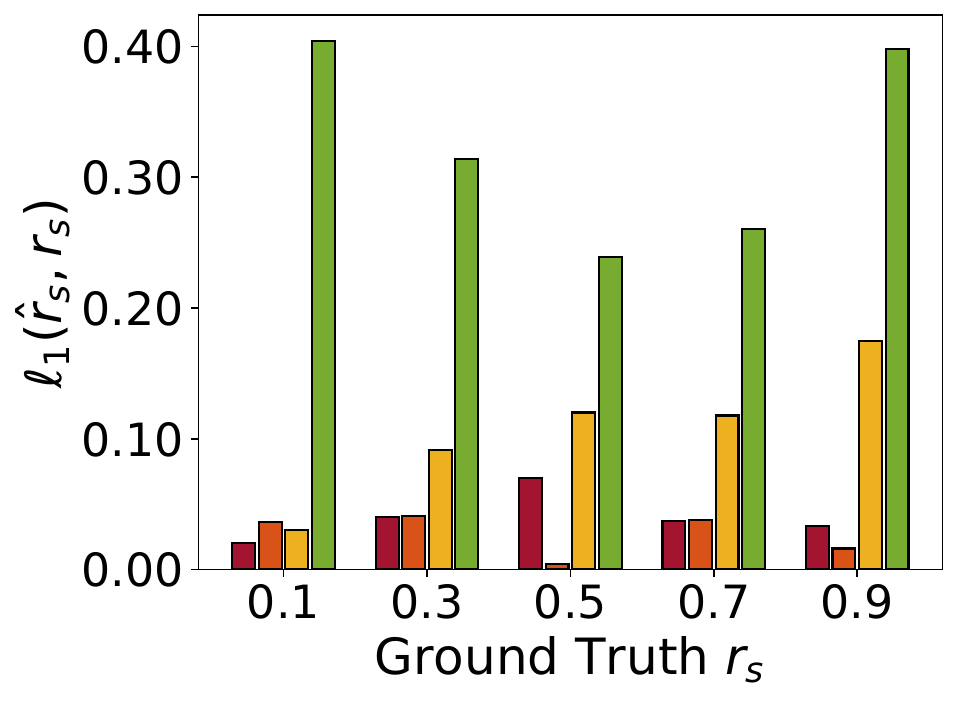}\label{subfig-PIA-diff-ratio-withaux}}
&
\hspace{-5mm}
\subfloat[$\mathcal{A}^{(2)}_{\text{PIA}}$]{\includegraphics[width=0.48\columnwidth]{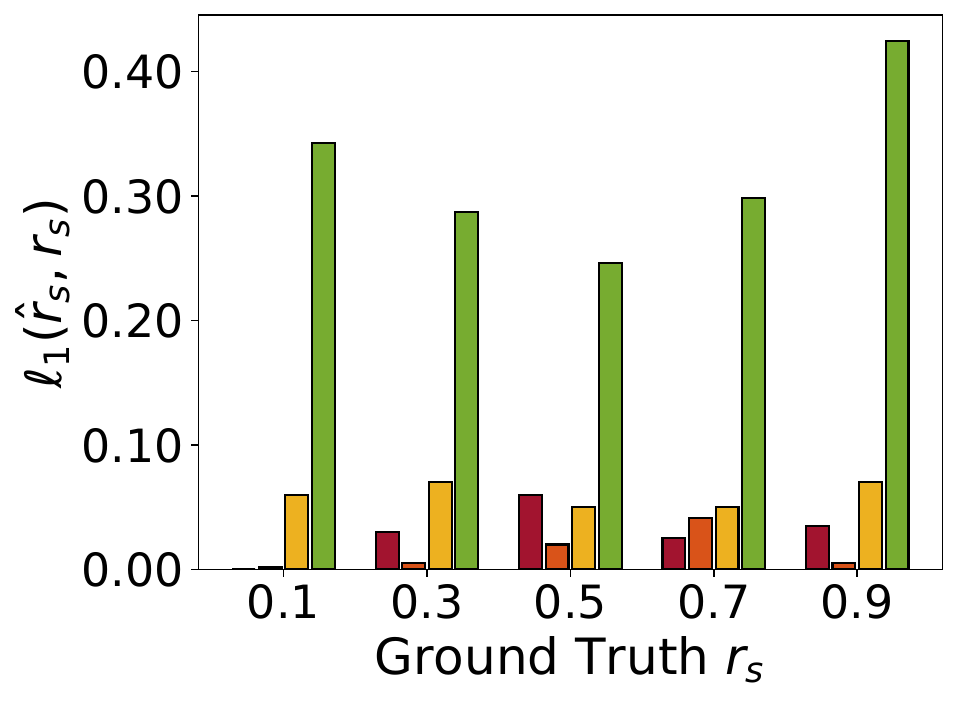}\label{subfig-PIA-diff-ratio-noaux}}
\end{tabular}
\caption{The results of PIA attacks \textit{w.r.t.} different $r_s$.}
\label{fig-PIA-diff-ratio}
\end{small}
\end{figure}

\noindent
{\textbf{Attack Performance \textit{w.r.t.} Different Generative Models.}
Now we compare PIA's attack performance on different generative models.
We conducted experiments with a fixed sample size of 200 and $r_s=0.5$, and trained different generative models on different datasets to evaluate the performance of PIA. 
The results are summarized in Table~\ref{tb-comp-PIA-model}.
The main observation from Table~\ref{tb-comp-PIA-model} is that PIA achieves more accurate attack results on diffusion models than on GAN and VAE models. 
This is attributed to the better distribution coverage and sampling quality of diffusion models compared to other types of models, as discussed in Section~\ref{subsec-PIA-details}.
In addition, we observe that the attacks achieve the largest error when tested on SNGAN trained on MNIST.
The reason is that \textit{mode collapse} exists in this SNGAN, i.e., $80\%$ of the generated samples are of number 0, which is far from the ground truth $r_s=0.5$.
%

In addition, the experiments regarding the impact of model underfitting on our attacks are deferred to Appendix~\ref{appen-exp-underfitting}.

\begin{table}[t!]
\setlength{\tabcolsep}{1.2pt}
\caption{Attack comparison of PIA \textit{w.r.t.} different models.}\label{tb-comp-PIA-model}
\vspace{-1mm}
\center
\begin{small}
\begin{tabular}{c|c|ccc||cc||c}
\hline
Dataset &Attack & DDPM & NCSN & SDEM & SNGAN & SSGAN & RHVAE\\
\hline 
\hline
\multirow{2}{*}{MNIST}      & $\mathcal{A}^{(1)}_{\text{PIA}}$ & 0.043	& 0.073   & 0.053 
    &	0.292  &	0.045  & 0.136  \\
                            & $\mathcal{A}^{(2)}_{\text{PIA}}$ & 0.012  &	0.060 & 
0.085	&  0.203  &	 0.095 & 0.131    \\                      
\hline
\multirow{2}{*}{CelebA}    & $\mathcal{A}^{(1)}_{\text{PIA}}$  & 0.056     &	0.004 &	0.024	&  0.114   &    0.091  & 0.201 \\
                            & $\mathcal{A}^{(2)}_{\text{PIA}}$ & 0.042     &	0.020 &	
0.012	&  0.166  & 	0.129  & 0.189  \\    
\hline
\multirow{2}{*}{AFAD}       & $\mathcal{A}^{(1)}_{\text{PIA}}$ & 0.055  &	0.083	  & 0.102	&     0.144	   &      0.100   & 0.179 \\
                            & $\mathcal{A}^{(2)}_{\text{PIA}}$ & 0.107  &	0.040 &	
0.087   &	 0.098   &   0.123   & 0.200     \\                              
\hline
\end{tabular}
\end{small}
\end{table}

\section{Discussion}\label{sec-discussion}

\noindent
\textbf{Connection between PIA and FPA.} 
In this paper, we explore the inherent risks associated with the prevalent data-sharing pipeline via diffusion models through a fairness poisoning attack on the sharer's side and a property inference attack on the receiver's side. 
Note that the PIA and FPA both fall within the category of distribution-based attacks.
In an untrusted collaborative environment~\cite{ndssByzantine, securityByzantine}, these two attacks can be employed as countermeasures against each other. Specifically, the PIA can be used by the receiver to audit the distribution of $\dhat$ to eliminate possible bias, while the FPA can also be used by the sharer to change the distribution of sensitive properties for defending against the inference attack initiated by the receiver.
Additional details regarding these countermeasures are deferred to Appendix~\ref{appendix-sec-countermeasure}.
By presenting PIA and FPA together, we derive two significant findings for the data-sharing pipeline: first, it becomes apparent that both parties involved in the data-sharing process may be exposed to risks. This underscores the necessity for the implementation of secure and ethical protocols to safeguard both parties in real-world applications; second, we recognize that both FPA and PIA are inherently interconnected, and their performance can potentially be influenced by each other (Appendix~\ref{appendix-sec-countermeasure}). This observation can offer some insights for designing robust data-sharing mechanisms in future research.
%
%
We believe that the proposed attacks can serve as catalysts for further exploration of generalized and practical defense mechanisms against distribution-based attacks in future studies.

\noindent
\textbf{Attack Limitation.} 
The generalization of the proposed attacks to other generative models, such as GANs and VAEs, may be relatively limited, as the performance of our attacks depends on the distribution coverage feature of generative models. 
However, it is important to highlight that practical data-sharing applications require generative models with robust distribution coverage. Inadequate sampling density, as observed in GANs and VAEs, can compromise data utility, making them unsuitable for data-sharing purposes~\cite{healthRecordSharing}. 
Recent studies~\cite{bond2021deep} have demonstrated that diffusion models surpass other types of generative models in both sampling quality and distribution coverage, rendering them the most favored choice in recent data-sharing investigations~\cite{augStableDiffusion, healthRecordSharing, zeroshoticlr, sharingHistopathology, zeroshotdiversity}. Consequently, 
while our attacks may not exhibit broad generality across all types of generative models, they are still suitable for practical and promising data-sharing scenarios.


\section{Related Work}\label{sec-related-work}
\noindent
\textbf{Generative Models.} 
The recent developments in generative models can be classified into three main categories, namely likelihood-based methods~\cite{nice, likelihoodPixel, VAE}, GAN-based methods~\cite{GAN, SSGAN, SNGAN}, and diffusion models~\cite{DDPM, SDE, NCSN, stableDiffusion}.
%
%
A comparative study~\cite{beatgan} between GANs and diffusion models has demonstrated that the latter can generate images with comparable visual qualities to those produced by the state-of-the-art GAN models while offering better distribution coverage and being easier to train.
Given these advantages, diffusion models have been widely used in different domains, such as zero-shot learning~\cite{zeroshotdiversity, zeroshoticlr} and data sharing~\cite{sharingMammography, sharingMedicalImages, sharingHistopathology, sharingRadiograph}.
%
%
Despite their potential benefits, the risks associated with sharing diffusion models have not been fully explored.
%

\noindent
\textbf{Inference Attacks.}
Machine learning algorithms are susceptible to inference attacks, which can be broadly classified into property inference~\cite{jiang2024data, propertyFL, propertyGAN, propertyMembership, zhu2023passive}, membership inference~\cite{membershipFL,membershipMLLeaks,ShokriSSS17}, and feature inference attacks~\cite{luo2021feature,luo2021fusion,luoccs,featureGAN,featureModelInversion}.
While membership inference and feature inference attacks aim to reveal the record-level privacy of the private training datasets, property inference attacks seek to obtain group-level information about these   datasets.
Recent studies have also proposed membership inference~\cite{diffusionFeature, diffusionMembership1, diffusionMembership2} on diffusion models and property inference on GANs~\cite{propertyGAN}, but these attacks mainly require white-box access to the pre-trained models~\cite{diffusionFeature, diffusionMembership1, diffusionMembership2, propertyGAN}.
%
%
In this paper, we focus on property inference attacks against diffusion models, in a more practical and restricted setting, i.e., black-box attacks.
Our proposed attack is lightweight and exhibits high inference accuracy, as demonstrated in our experiments. 

\noindent
\textbf{Poisoning Attacks.}
Different from inference attacks, poisoning attacks require the adversary to actively modify the training datasets for changing the behaviors of downstream models.
Poisoning attacks can be commonly classified as backdoor injection~\cite{backdoor1,backdoorFL,backdoorOnePixel,backdoorDiffusion}, accuracy poisoning~\cite{accuracyPoison1,accuracyPoison2,accuracyPoison3}, and fairness poisoning~\cite{fp-accuracy,pf-chang,pf-dasfaa,pf-ecml,pf-labelflip}.
In this paper, we focus on fairness poisoning. 
Existing research~\cite{fp-accuracy,pf-chang,pf-dasfaa,pf-ecml,pf-labelflip} on this type of attack generally aims to degrade both model accuracy and fairness by designing a loss function that incorporates both model loss and fairness gap, followed by selecting poisoning records from an auxiliary dataset based on this function.
However, since model losses are needed during the execution of attack algorithms, these attacks require white-box access to the model architecture and training algorithm, which is a relatively strong assumption.
In contrast, we present a novel approach to fairness poisoning that degrades fairness while preserving the accuracy of the target model. Compared to current studies, our approach is more practical and stealthy as it does not require access to the target model.

\section{Conclusion}\label{sec-conclusion}
In this paper, we examine the privacy and fairness vulnerabilities that arise in diffusion model sharing through an adversarial lens.
Specifically, we introduce a fairness poisoning attack at the sharer's side and a property inference attack at the receiver's side.  
%
%
Through extensive experiments across various diffusion models and datasets, we demonstrate the efficacy of the proposed attacks. 
%


\bibliographystyle{IEEEtran}
\balance
\bibliography{references.bib}

\clearpage
\appendix

\subsection{Proof of Theorem~\ref{thm-error-bound}}\label{appendix-proof-thm}
\begin{proof}
Algorithm~\ref{alg-attack-PIA} takes $\hat{m}$ samples from the diffusion model to construct the synthetic dataset $\hat{\mathcal{D}}_{y=0}$.
Accordingly, the formation of $\hat{\mathcal{D}}_{y=0}$ is equivalent to independently drawing $\hat{m}$ samples of $X$, or, sampling $\hat{m}$ independent random variables identical to $X$ for once.  
Abusing the notation a bit, we use $S\in\{0,1\}$ to represent the indicator random variable that characterizes the presence of the sensitive feature $S$ over the samplings on $X$, $S$ itself is also a random variable associated (correlated) to $X$.
Let $\hat{S}_i$ be the estimates of $S_i$ \textit{w.r.t.} each $i\in \{1,\cdots,\hat{m}\}$.
The true proportion of $s$ (i.e., $S=1$) based on $\hat{\mathcal{D}}_{y=0}$ is denoted by 
\begin{equation}\label{eq-hoeffding-pre00}
    \hat{r}^*_s = \frac{1}{\hat{m}}\cdot\sum_{i=1}^{\hat{m}} S_i. 
\end{equation}
%
The ground truth proportion of $s$ is denoted by
\begin{equation}\label{eq-hoeffding-pre01}
    r^*_s = \expt[S].
\end{equation}
From $\{\hat{S}_i:i\in\{1,\cdots,\hat{m}\}\}$ 
we obtain an estimate of $r^*_s$:  
\begin{equation}\label{eq-hoeffding-pre02}
    \hat{r}_s = \frac{1}{\hat{m}}\cdot \sum_{i=1}^{\hat{m}} \hat{S}_i.
\end{equation}
Based on the fact that the prediction error of $g_{\text{d}}$ is $\epsilon_d$, we have
\begin{equation}\label{eq-d-error}
    |\hat{r}^*_s - \hat{r}_s|\leq \epsilon_d.
\end{equation}
As $g_{\text{d}}$ is unbiased by our assumption, we have $\expt[\hat{S}] = \expt[S]$ which implies 
\begin{equation}\label{eq-hoeffding-pre03}
    \expt\left[\sum_{i=1}^{\hat{m}} \hat{S}_i\right] = \expt\left[\sum_{i=1}^{\hat{m}} S_i\right]=\hat{m}\cdot\expt[S] = \hat{m}\cdot r^*_s.
\end{equation}
Accordingly, for any $t\geq 0$, Eq.~\eqref{eq-hoeffding-pre00}, \eqref{eq-hoeffding-pre03}, and Hoeffding's inequality~\cite{hoeffding} implies
\begin{equation}\label{eq-hoeffding-orig}
\begin{aligned}
    \prob\left(\left|\hat{m}\cdot \hat{r}^*_s - \hat{m}\cdot r^*_s\right| \geq t\right)
     &=  \prob\left(\left|\sum_{i=1}^{\hat{m}} S_i - \expt\left[\sum_{i=1}^{\hat{m}} S_i\right]\right| \geq t\right)\\
    &\leq  2\exp{\left(-\frac{2t^2}{\hat{m}} \right)}.
\end{aligned}
\end{equation}
{
Taking $t =\epsilon\cdot \hat{m} $ for some $\epsilon\geq 0$, we have
\begin{equation}\label{eq-hoeffding-sim}
    \prob\left(\left|\hat{r}^*_s - r^*_s \right|\geq \epsilon\right) \leq 2 \exp\left(-2\hat{m}\epsilon^2\right).
\end{equation}
}
%
%
By combining Eq.~\eqref{eq-hoeffding-sim} and Eq.~\eqref{eq-d-error}, we obtain 
\begin{equation}
    \prob\left(\left|\hat{r}_s - r^*_s \right|\geq \epsilon+\epsilon_d\right) \leq 2 \exp\left(-2\hat{m}\epsilon^2\right).
\end{equation}

\end{proof}

\subsection{Experimental Setting}\label{appen-exp-setting}

We implement the proposed attacks in Python with \textit{PyTorch}\footnote{https://pytorch.org/} and conduct all experiments on a server equipped with AMD EPYC 7313 16-Core Processor $\times$64, an NVIDIA RTX A5000, and 503GB RAM, running Ubuntu 20.04 LTS.

\vspace{0.5mm}
\noindent
\textbf{Attack Implementation.} 
In the fairness poisoning attack, we employ a multilayer perceptron (MLP) as the receiver's downstream classifier.
The MLP comprises two successive modules of one convolutional layer plus one pooling layer, followed by three fully connected layers with the Rectified Linear Unit (ReLU) activation function.
For Algorithm~\ref{alg-FPA-greedy}, we utilize MINE~\cite{MINE} as the mutual information estimator $\phi$. The size $m_p$ of $\dpoi$ is fixed to 6000.
In the property inference attack, we use an MLP with the same structure as in FPA as the property discriminator of $\mathcal{A}^{(1)}_{\text{PIA}}$. 
For $\mathcal{A}^{(2)}_{\text{PIA}}$, we use the pretrained CLIP model\footnote{{https://github.com/openai/CLIP}} as the property discriminator.

\vspace{0.5mm}
\noindent
\textbf{Datasets.} 
Following prior research on property inference and fairness poisoning~\cite{propertyPoison, propertyGAN, changbias}, we use three image datasets and one tabular dataset in the experiments. The image datasets are \textit{MNIST}~\cite{mnist}, CelebFaces Attributes (\textit{CelebA})~\cite{celeba}, and the Asian Face Age Dataset (\textit{AFAD})~\cite{AFAD}. The tabular dataset is \textit{Adult}~\cite{UCI}, which consists of 14 features for predicting whether one's income will exceed 50000 per year.
For consistency, we resize all training images to $32\times 32$ before the experiments and use 6000 images~\footnote{The training datasets consisting of 6000 images are sufficient for our experiments, because we have tried the training sizes of 4000, 6000, 8000, and 10000 in our experiments and observed that the attack results remain consistent across these different training sizes. Consequently, we opt to utilize a training size of 6000 images.} to train diffusion models.
Note that in the property inference attack, the target property gender ($\{\text{male}, \text{female}\}$) in CelebA and age group ($\{[18, 20], [30, 39]\}$) in AFAD are binary properties, while the property number ($\{0,1,\cdots,$ $9\}$) in MNIST and race ($\{${White}, {Asian-Pac-Islander}, {Amer-Indian-Eskimo}, {Other}, {Black}$\}$) in Adult are properties with multiple values.
For the AFAD datasets, we randomly split the training images into two groups, labeling one group with class 0 and the other with class 1.
In the fairness poisoning attack, we pre-filter the datasets to obtain a dataset $\dbig^o$ in which all the partitions $\{\dbig^o_{s=j,y=k}$$, (j,k)\in \mathcal{S}\times \mathcal{Y}\}$ have the same size, i.e., $\dbig^o$ follows a clean and fair distribution. We further randomly sample $80\%$ records in $\dbig^o$ as the sharer's private dataset $\dbig$ and the remaining records as the receiver's test dataset $\dtest$.

\begin{table*}[!th]
\setlength{\tabcolsep}{7pt}
\caption{Performance comparison of FPA \textit{w.r.t.} different proportions of real images in few-shot learning.}\label{tb-comp-few-shot}
\center
\small
\begin{tabular}{c|c|cc||cc||cc||cc||cc}
\hline
\multirow{2}{*}{Dataset} &\multirow{2}{*}{Loss} & \multicolumn{2}{c||}{10\%} & \multicolumn{2}{c||}{20\%} & \multicolumn{2}{c||}{30\%} & \multicolumn{2}{c||}{40\%} & \multicolumn{2}{c}{50\%} \\
\cline{3-12}
  &   & FPA  & LF & FPA & LF & FPA &
   LF  & FPA  & LF & FPA & LF  \\
\hline 
\hline
\multirow{2}{*}{CelebA-A}   & $\ell_{\text{acc}}$ & 3.25\%    & 13.26\%& 3.31\%   & 10.52\%  &	2.99\%  &	 10.37\% & 2.70\% &  7.01\% & 1.56\%  & 6.77\% \\
                            & $\ell_{\text{fair}}$ & 0.2639  & 0.4431 & 0.2301  & 0.3395   & 	0.1657	& 0.3004 & 0.1261 & 0.2183 &  0.1031 &   0.2038 \\                      
\hline
\multirow{2}{*}{CelebA-S}    & $\ell_{\text{acc}}$ & 4.07\%  & 19.55\%& 3.58\%  & 16.56\% &		1.43\%&	12.97\% & 0.33\% & 9.58\% & 0.08\% &  7.89\%  \\
                            & $\ell_{\text{fair}}$ & 0.2682	 &0.4346 & 0.2017  &0.3605   & 	0.1590	& 0.3265 & 0.1223 & 0.2478 &  0.0988 &   0.1430 \\                       
\hline
\end{tabular}
\end{table*}

\subsection{The Attack Performance under Few-Shot Learning}\label{appendix-few-shot}

Data sharing is typically advocated within data-scarce settings~\cite{zeroshoticlr}, where zero-shot learning is widely utilized in various data-sharing studies~\cite{zeroshoticlr, healthRecordSharing, TabDDPM, zeroshotdiversity}.
In addition, several studies~\cite{augStableDiffusion, fewshotStableDiffusion} also explore few-shot learning in data-sharing applications, wherein synthetic data is combined with limited real images to enhance model robustness. To validate the efficacy of our fairness poisoning attack within the context of few-shot learning, we conducted supplementary experiments employing DDPM as the experimental model. The poisoning proportion $\alpha$ was set to 0.5. We varied the proportions of real images within the composed training datasets at 10\%, 20\%, 30\%, 40\%, and 50\% to perform few-show learning at the receiver's side and examine the effects of our fairness poisoning attack. The losses of accuracy and fairness are shown in Table~\ref{tb-comp-few-shot}.

These results demonstrate the continued effectiveness of our attack in the context of few-shot learning. Even in scenarios where real images account for 50\% of the dataset, the fairness poisoning attack is still capable of compromising the fairness of downstream classifiers to some extent.

\subsection{The Impact of Model Underfitting}\label{appen-exp-underfitting}
Given that the effectiveness of the proposed attacks is contingent on the distribution coverage of diffusion models, where $\prob_{\dbig}\approx \prob_{\dhat}$, one natural question is whether or not the performance of the proposed attacks can be impacted by underfitted diffusion models.
In other words, we need to examine whether distribution coverage remains robust in underfitted diffusion models. 
To address this question, we perform PIA on underfitted models using CelebA as the experimental dataset, with a ground truth property proportion of $r_s=0.5$. The sampling size of PIA is set to 1000. 
Considering that the training epoch used in Section~\ref{sec-exp} for DDPM, NCSN, and SDEM is 1500, we train underfitted models with varying training epochs among $\{600,	800, 1000, 1200, 1400\}$. For reference, the training losses for DDPM at different training epochs are 0.123, 0.088, 0.076, 0.049, and 0.037. 
The $\ell_1$ losses of PIA performed on those underfitted models are summarized in Table~\ref{tb-PIA-underfitting}.
One main observation from Table~\ref{tb-PIA-underfitting} is that the $\ell_1$ losses remain relatively consistent when the training epoch exceeds 800, indicating that distribution coverage remains intact in moderately underfitted diffusion models.
When the training epoch falls below 600, the attack losses become relatively larger. The reason is that diffusion models under such training epochs are excessively underfitted, resulting in synthetic images with less distinguishable visual features.
In summary, moderate underfitting in diffusion models has a negligible impact on the proposed attacks, while severe underfitting can relatively degrade the attack performance. However, severe underfitting can also harm model utility and make synthetic images unsuitable for data-sharing applications.

\begin{table}[t!]
\setlength{\tabcolsep}{6pt}
\caption{Comparison of PIA $\ell_1$ losses \textit{w.r.t.} different training epochs.}\label{tb-PIA-underfitting}
\center
\begin{small}
\begin{tabular}{c|c|ccccc}
\hline
Dataset &Attack & 600 & 800 & 1000 & 1200 & 1400\\
\hline 
\hline
\multirow{2}{*}{DDPM}      & $\mathcal{A}^{(1)}_{\text{PIA}}$ & 0.001&	0.026&	0.001&	0.049&	0.023  \\
                            & $\mathcal{A}^{(2)}_{\text{PIA}}$ & 0.036&	0.042&	0.021&	0.047&	0.025   \\                      
\midrule
\multirow{2}{*}{NCSN}    & $\mathcal{A}^{(1)}_{\text{PIA}}$  & 0.078&	0.004&	0.029&	0.045&	0.004 \\
                            & $\mathcal{A}^{(2)}_{\text{PIA}}$ & 0.069&	0.020&	0.001&	0.035&	0.023  \\    
\midrule
\multirow{2}{*}{SDEM}       & $\mathcal{A}^{(1)}_{\text{PIA}}$ & 0.044&	0.016&	0.001&	0.010&	0.023 \\
                            & $\mathcal{A}^{(2)}_{\text{PIA}}$ & 0.047&	0.021&	0.032&	0.033&	0.023    \\                              
\midrule
\end{tabular}
\end{small}
\vspace{-2mm}
\end{table}

\subsection{Possible Countermeasures}\label{appendix-sec-countermeasure}
\noindent
\textbf{Differential Privacy}.
Differential privacy (DP) is a well-established privacy-preserving mechanism that has been applied to various applications, including generative models~\cite{dpGAN, dpDiffusion} and data publishing~\cite{dpWavelet}, due to its strong theoretical guarantee.
However, the proposed PIA cannot be effectively prevented by DP to protect the private distribution of $\dbig$.
DP is typically utilized to protect record-level privacy~\cite{dpWavelet, dpDiffusion}. In contrast, the proposed PIA aims to infer group-level privacy, which is difficult to defend using DP, as demonstrated in \cite{propertyPermutation, propertyPoison}. 
One possible approach is to use group-level DP to protect the distribution privacy of $\dbig$, as introduced in \cite{jiang2024protecting, dpClient}. But this requires injecting a significant amount of noise during the training of diffusion models, which can considerately degrade the generative performance of diffusion models.
Further exploration is required for techniques that are capable of simultaneously preserving the generative performance of diffusion models and safeguarding the group-level privacy of training datasets.

\vspace{0.5mm}
\noindent
\textbf{Data Re-Sampling}.
To defend against possible property inference attacks initiated by the receiver, the sharer may consider leveraging the spirit of FPA and modifying the optimization objective described in Eq.~(\ref{eq-opt-sim}).
For example, if the sharer wants to hide the proportion of an important property $s$ in $\dbig$, the objective in Eq.~(\ref{eq-opt-sim}) can be changed to maximizing $|r_{s\in\dbig}-r_{s\in\dpoi}|$.
By doing so, the sharer can ensure that the property proportion estimated by the receiver $r_{s\in\dpoi}$ is significantly different from the ground truth $r_{s\in\dbig}$, while causing minimal harm to the utility of the receiver's models.
Note that the data utility constraint in Eq.~(\ref{eq-opt-sim}) is important in fostering trust in the collaboration between the sharer and the receiver.
Simply maximizing $|r_{s\in\dbig}-r_{s\in\dpoi}|$ can lead to low prediction performance on the receiver's side, ultimately harming the benefits of both parties in the long run.

To evaluate the efficacy of data re-sampling in mitigating property inference attacks, we choose CelebA with the sensitive property gender ``male'' and Adult with the sensitive property race ``white'' as the experimental datasets. The experimental setup is identical to that used in Fig.~\ref{fig-PIA-diff-ratio}. 
The main difference in this experiment is that prior to training diffusion models, the sharer employs a re-sampling algorithm to alter the distribution of protected properties. This algorithm is derived from Algorithm~\ref{alg-FPA-greedy}, with lines \ref{alg-re-sample-start} - \ref{alg-maxSY-end} modified to identify the subset of records that maximizes $|r_{s\in\dbig}-r_{s\in\dpoi}|$.
The results of property inference attacks after data re-sampling are depicted in Fig.~\ref{fig-defense-PIA}, while the corresponding accuracy degradation caused by the re-sampling algorithm is presented in Table \ref{tb-defense-PIA}. 
Note that the PIA performed on Adult needs no property discriminators. Nonetheless, we depict the results on Adult in Fig.~\ref{subfig-defense-PIA-noaux} for a direct comparison.
As evident from Fig.~\ref{fig-defense-PIA}, the application of data re-sampling techniques can diminish the effectiveness of property inference attacks, making them comparable to random guess. Simultaneously, this approach inflicts minimal harm on data utility and downstream classifiers, as demonstrated in Table \ref{tb-defense-PIA}.

\vspace{0.5mm}
\noindent
\textbf{Data Auditing}. 
Correspondingly, the receiver can also utilize PIA to audit the sampled dataset $\dhat$ and detect any training bias introduced by the sharer.
Specifically, the receiver can select a set of sensitive features and use PIA to infer their distribution histograms in $\dhat$.
Based on these histograms, the receiver can identify biased features.
For example, in the generated CelebA dataset, the gender ``male'' may account for only 0.1 of the samples (denoted as ${r}_s=0.1$ in Fig.~\ref{fig-PIA-diff-ratio}).
Then, the receiver can employ a randomized sampling method~\cite{biasSampling} to rebalance the distribution before model training. 
Additionally, PIA can be used to audit the original dataset $\dbig$ for any \textit{unintentional bias} injected during the data collection phase~\cite{fairnessDP, fairnessEO}.
As described in Section~\ref{subsec-error-bound}, even with a relatively small sample size of 100, PIA can be used to effectively identify potential bias in the datasets.
Therefore, PIA can serve as a useful practice for data auditing and ensuring the fairness of the training process in various settings.

\begin{figure}[t!]
\centering
\begin{small}
\begin{tabular}{cc}
\multicolumn{2}{c}{\hspace{0mm} \includegraphics[height=7mm]{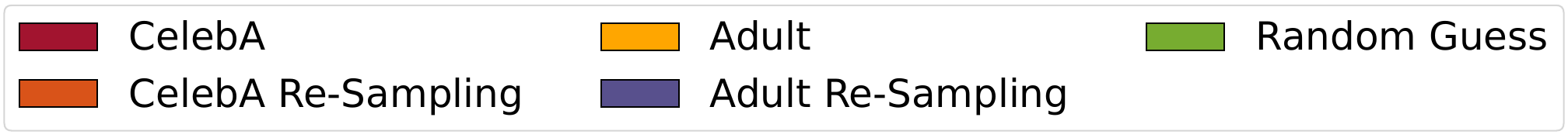}}
\vspace{-4mm}  \\
\hspace{-4mm}
\subfloat[$\mathcal{A}^{(1)}_{\text{PIA}}$]{\includegraphics[width=0.50\columnwidth]{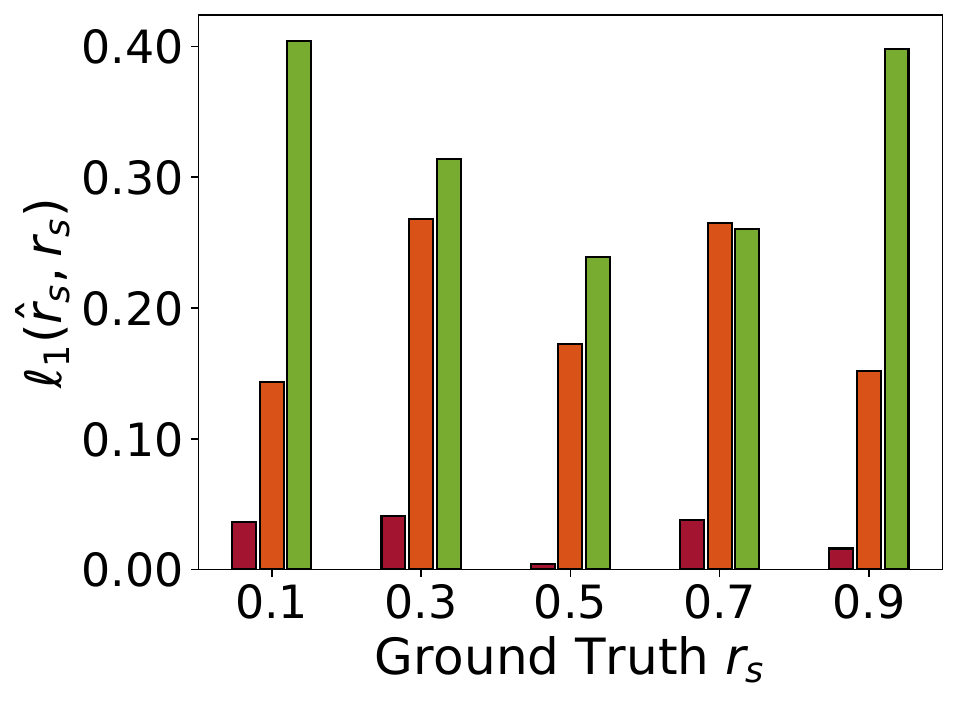}\label{subfig-defense-PIA-withaux}}
&
\hspace{-6mm}
\subfloat[$\mathcal{A}^{(2)}_{\text{PIA}}$]{\includegraphics[width=0.50\columnwidth]{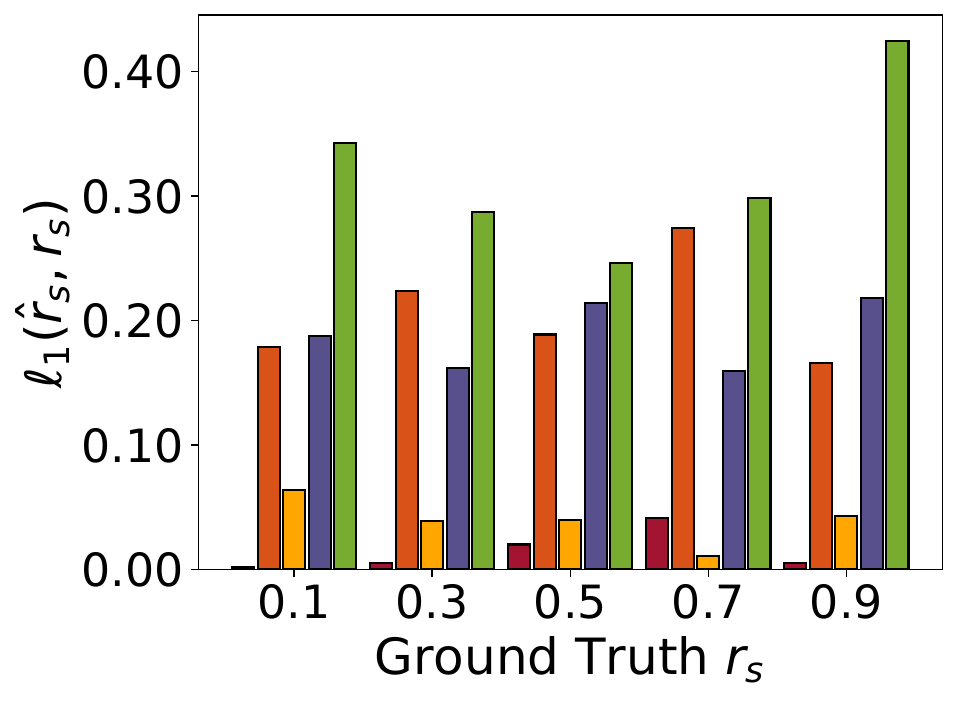}\label{subfig-defense-PIA-noaux}}
\end{tabular}
\caption{The results of PIA after data re-sampling.}
\label{fig-defense-PIA}
\end{small}
\end{figure}

\begin{table}[t!]
\setlength{\tabcolsep}{7.5pt}
\caption{The accuracy degradation $\ell_{\text{acc}}$ after employing data-resampling.}\label{tb-defense-PIA}
\center
\begin{small}
\begin{tabular}{c|ccccc}
\hline
\multirow{2}{*}{Dataset} & \multicolumn{5}{c}{Ground Truth $r_s$} \\
\cline{2-6}
& 0.1 & 0.3 & 0.5 & 0.7 & 0.9\\
\hline 
\hline
{CelebA}    & 2.15\%	&1.51\%	&0.59\%	&3.49\%	&4.08\%  \\ 
{Adult}     & 5.05\%	&1.64\%	&1.30\%	&4.00\%	&4.89\%    \\                       
\hline
\end{tabular}
\end{small}
\end{table}

\end{document}